\documentclass[12pt]{article}

\usepackage{geometry}
\usepackage[utf8]{inputenc} 
\usepackage[T1]{fontenc}    
\usepackage{hyperref}       
\usepackage{url}            
\usepackage{booktabs}       
\usepackage{amsfonts}       
\usepackage{nicefrac}       
\usepackage{microtype}      
\usepackage{wrapfig}

\usepackage{blindtext}

\usepackage{fullpage}
\usepackage[round, authoryear]{natbib}
\usepackage{enumitem}
\usepackage{setspace} 
\usepackage{times}
\usepackage{algorithm}
\usepackage{algorithmic}
\usepackage{amssymb}
\usepackage{amsmath}
\usepackage{amsthm}
\usepackage{color}
\usepackage{graphicx} 
\usepackage{authblk}
\usepackage{hhline}
\usepackage{enumitem}
\usepackage[FIGTOPCAP]{subfigure}

\newcounter{ass_counter}
\newtheorem{theorem}{Theorem}[section]
\newtheorem{assumption}[ass_counter]{Assumption}
\newtheorem{lemma}[theorem]{Lemma}

\def\x{\mathbf{x}}
\def\y{\mathbf{y}}
\def\w{\mathbf{w}}
\def\u{\mathbf{u}}
\def\y{\mathbf{y}}

\def\R{\mathbb{R}}

\def\P{\mathbf{Pr}}
\def\h{\mathbf{h}}

\RequirePackage{amsfonts}
\newcommand{\real}{\mathbb{R}}
\providecommand{\norm}[1]{\lVert#1\rVert}
\providecommand{\abs}[1]{\lvert#1\rvert}

\DeclareMathOperator*{\argmin}{argmin}
\newcommand{\E}{\mathrm{E}}

\newcommand{\z}{\mathbf{z}}

\newcommand{\diag}{\mathrm{diag}}

\newcommand{\bd}{\boldsymbol}

\providecommand{\tabularnewline}{\\}
\newcommand\given[1][]{\,#1\vert\,}

\expandafter\def\expandafter\normalsize\expandafter{%
    \normalsize
    \setlength\abovedisplayskip{3pt}
    \setlength\belowdisplayskip{3pt}
    \setlength\abovedisplayshortskip{3pt}
    \setlength\belowdisplayshortskip{3pt}
}

\usepackage{titlesec}
\titleformat{\section}[block]{\bfseries\filcenter}{\thesection.}{1em}{}
\titleformat{\subsection}[hang]{\bfseries}{\thesubsection.}{1em}{}
\titleformat{\section}[block]{\bfseries\filcenter}{\thesection.}{1em}{}

\title{An Interactive Greedy Approach to Group Sparsity in High Dimensions}

\author[a]{\vspace{-.2in}  Wei Qian \footnote{Contributed equally.}}
\newcommand\CoAuthorMark{\footnotemark[\arabic{footnote}]} 
\author[b]{Wending Li \protect\CoAuthorMark
}
\author[c]{Yasuhiro Sogawa
}
\author[c]{Ryohei Fujimaki
}
\author[b]{Xitong Yang
}
\author[b]{Ji Liu
}

\affil[a]{{\small Department of Applied Economics and Statistics, University of Delaware, \texttt{weiqian@udel.edu}}}
\affil[b]{{\small Department of Computer Science, University of Rochester,
\texttt{wending.li@rochester.edu;yangxitongbob@gmail.com; ji.liu.uwisc@gmail.com}}}
\affil[c]{\small{ Knowledge Discovery Research Laboratories, NEC, USA,
  \texttt{\small{y-sogawa@ah.jp.nec.com;rfujimaki@nec-labs.com}}}}

\date{}

\begin{document}
\maketitle
\vspace{-.5in}
\begin{abstract} 
 Sparsity learning with known grouping structure has received  considerable attention due to wide modern applications in high-dimensional data analysis. Although advantages of using group information have been well-studied by shrinkage-based approaches, benefits of group sparsity have not been well-documented for greedy-type methods, which much limits our understanding and use of this important class of methods. In this paper, generalizing from a popular forward-backward greedy approach, we propose a new interactive greedy algorithm for group sparsity learning and prove that  the proposed greedy-type algorithm attains the desired benefits of group sparsity under  high dimensional settings. An estimation error bound refining other existing methods and a guarantee for group
 support recovery are also established simultaneously. In addition, we
 incorporate a general M-estimation framework and introduce an
 interactive feature to allow extra algorithm
 flexibility without compromise in theoretical properties. The
 promising use of our proposal is demonstrated through
 numerical evaluations including a real industrial application in
 human activity recognition at home. Supplementary materials for this article are available online.
\end{abstract} 

\noindent
{\small {\bf Key Words:} group
Lasso, smart home, stepwise selection, variable selection, human activity recognition}

\newpage
\section{Introduction}
\label{sec:intro}

High-dimensional data have become prevalent in many modern statistics
and data mining applications. 
Given feature vectors $\x_1, \x_2,\cdots, \x_n\in \real^p$ and a
sparse coefficient vector $\w^*\in\real^p$, consider a standard
sparsity learning
setting that 
$y_i\in\real$ ($i=1,\cdots,n$) is independently generated from some distribution
$P_{\theta_i}$  depending on $\theta_i$ with linear relation
$\theta_i=\x_i^T\w^*$, and that the feature dimension $p$ is much larger
than  $n$. Among various sparsity scenarios, we focus on  group sparsity, where elements in $\w^*$
can be partitioned into known groups, and coefficients in each group are
assumed to be either all zeros or all nonzeros. It is  of primary
interest to provide accurate estimation of $\w^*$ and identify the
relevant groups. If ${\rm supp}(\w^*)$ -- the index set of nonzero
elements in $\w^*$ -- were known, the targeted problem can be naturally formulated through a criterion function $Q$ to find the ``idealized'' M-estimator 
\begin{equation}
\begin{aligned}
\bar\w = \argmin_{\substack{\w \in \mathbb{R}^p:\\{\rm supp}(\w) {\subset} {\rm supp}(\w^*) }} \Bigl\{Q(\w) :=
  {\frac{1}{n}}\sum_{i=1}^n f_i(\x_i^T\w)\Bigr\},
\label{eq:groupl0}
\end{aligned}
\end{equation}
where  $f_i(\cdot)$ is a given loss function determined by $y_i$. 
For example, under a standard  linear model, we may choose least
square loss $f_i(\theta_i)=(y_i-\theta_i)^2$; under generalized linear
models (GLM)  with  canonical link \citep{mccullagh2000generalized},
we may use the negative log-likelihood. The estimator $\bar\w$ is
``idealized'' because ${\rm supp}(\w^*)$ is not known {\it a
  priori}. 

In certain real-world applications, the criterion function $Q$ may go
 beyond the standard setting 
of \eqref{eq:groupl0} when response $y_i$'s are not necessarily
independent but the coefficient $\w$ has a known grouping structure
requiring variable selection. In our theoretical and simulation studies (to be shown in
Section~\ref{sec:mainresult} and Section~\ref{sec:exp}), we
will focus
on the standard setting (with independent response). To allow
extended real-world applications, the criterion function $Q$ can take
more general forms (e.g., by setting $Q$ to be the
negative log-likelihood), and our
application example to be exihibited in Section~\ref{sec:app} will fall
under the extended setting (without independent response).

In standard sparsity learning, existing methods can be roughly
categorized into shrinkage-type approaches
\citep{tibshirani1996regression,zou2006adaptive,candes2005decoding,fan2001variable,zhang2010nearly}
and stepwise greedy-type approaches \citep{tropp2004greed,
 zhang2011adaptive,ing2011stepwise}.  Following abundant work on
sparsity learning,  group sparsity
has received considerable attention. Initially motivated from
multi-level ANOVA models and additive models \citep{yuan2006model},
group sparsity learning and related studies has seen its use in many practical
applications including genome wide association studies (e.g.,
\citealp{zhou2010association}), neuroimaging analysis (e.g.,
\citealp{jenatton2012multiscale}, \citealp{jiao2016group}), multi-task
learning (e.g., \citealp{lounici2011oracle}), actuarial risk
segmentation (e.g., \citealp{qian2016tweedie}), multi-view and manifold learning
(e.g., \citealp{culp2018data}), sufficient dimension reduction and envelope
models (e.g., \citealp{ding2018matrix,qian2018sparse}), among many
others. There also exists another
type of group sparsity for diversity selection
\citep{kong2014exclusive, yang2016benefits, huang2015exclusive} that emphasizes the diversity of selected features from groups. The benefits in estimation and inference using known grouping
structure were also demonstrated through the celebrated group-Lasso
methods
(see, e.g., \citealp{huang2010benefit,huang2012selective,lounici2011oracle,mitra2016benefit}
and references therein). 

In this paper,  we propose a new group sparsity learning algorithm
that generalizes from the practically very popular forward-backward
greedy-type approach \citep{zhang2011adaptive}.  We call this new
algorithm {\bf I}nteractive {\bf G}rouped Greedy {\bf A}lgorithm, or
IGA for abbreviation. 

As an increasingly popular application, the development of smart home management systems and
associated techniques (e.g., for automatic life-logging, emergency
alerts, energy control, etc.) has recently emerged as a promising applied research
area. One of the fundamentally important challenges is how to automatically
recognize human activities at homes. As a practical engineering
solution, multiple
pyroelectric sensors can be installed at home that capture binary signals in reaction to
human motion. Due to cost and other related issues, the 
machine learning task is to identify only a small number of deployed sensors while
maintaining reasonably high accuracy in human activity
classification. Interestingly, this application can be 
formulated into a group sparsity learning problem since features and
their coefficients that are associated with one sensor are naturally grouped
together, without overlapping with any other sensor groups. By
constructing the criterion function with negative log-likelihood under
the extended setting, we have applied the proposed IGA
algorithm for sensor (group) selection and will present the real data
study in Section~\ref{sec:app}.

\section{Related Work and Contribution}
To facilitate detailed exposition of our work's contribution, in
the following, we connect our work to existing
literature in Section~\ref{subsec:literature} before summarizing the contribution in Section~\ref{subsec:contribution}.

\subsection{Related Work}
\label{subsec:literature}

In group sparsity learning, it is assumed that coefficient
$\w=(w_1,w_2,\cdots,w_p)\in\real^p$ can be partitioned into $m$
($m\leq p$)
non-overlapping groups with a known grouping structure, and we use $G=\{1,2,\cdots,m\}$ to denote the group
index set. 
 An important and fruitful line of work
comes from shrinkage-based approaches. Extending from
the Lasso \citep{tibshirani1996regression}, the group Lasso has
received considerable attention. Define the
$\ell_{2,1}$ norm to be  $\|\w\|_{G, 1}:=\sum_{g\in G}\|\w_g\|$, where
$\w_{g}$ is the sub-vector of $\w$ corresponding to group $g$ and
$\norm{\cdot}$ is the  $\ell_2$ norm. Then the group Lasso
imposes the $\ell_{2,1}$ norm penalty on the criterion function $Q$ with a tuning parameter $\alpha$ by
\begin{equation}
\label{eq:grpLasso}
\min_{\w\in \R^p} Q(\w) + {\alpha}\|\w\|_{G, 1},
\end{equation}
and its methods and algorithms have been studied under both linear model
\citep{yuan2006model} and GLM (\citealp{kim2006blockwise};
\citealp{meier2008group}) settings. Mostly under
linear model settings, theoretical properties of the group Lasso on
both estimation and feature selection in high dimension
have been investigated in, e.g., \citet{nardi2008asymptotic} and \citet{wei2010consistent}. 
Notably, under certain variants of group restricted isometry property (RIP,
\citealp{candes2005decoding})  or restricted eigenvalue (RE,
\citealp{bickel2009simultaneous}) conditions, the group Lasso was shown
to be superior to the standard Lasso in estimation and prediction
\citep{huang2010benefit, lounici2011oracle}, thereby proving the
benefits of group sparsity.

 On the other hand, it is expected that the group Lasso inherits
 the same drawbacks of the Lasso, which include large estimation bias
 and relatively restrictive conditions on feature selection
 consistency \citep{fan2001variable}. Accordingly,  various
 non-convex penalties such as SCAD \citep{fan2001variable} and MCP
 \citep{zhang2010nearly} as well as the adaptive Lasso type penalties
 \citep{zou2006adaptive,zou2009adaptive, qian2013model}  were extended to allow incorporation of
 grouping structure and to replace the group Lasso penalty in
 \eqref{eq:grpLasso} so that consistent group selection can be
 achieved with less restrictive conditions
  than the group Lasso (see, e.g., \citealp{huang2012selective,jiao2016group} and references therein).

\begin{table} [htp!]
\caption{Comparison on the error bound $\|\hat{\w} - \w^*\|^2$ for
      sparse linear models. $\hat\w$
      is the estimator by each algorithm. Assume each group has equal
      size $q=p/m$.}
    \label{table:comparison}
\centering
\scalebox{0.9}{
    \begin{tabular}{ l c c c}
    \toprule
    Algorithm & Error Bound $O_p(\cdot)$ & Required sample size $\Omega(\cdot)$ & Reference \\ \midrule
    Lasso    &   $\frac{k_*\log(p)}{n}$  & $k_* \log (p)$ & \cite{van2011adaptive} \\ 
     Dantzig  &   $\frac{k_*\log(p)}{n}$  & $k_* \log (p)$ &\cite{candes2007dantzig}  \\ 
                    Multistage Lasso  & $\frac{k_*+\Delta_2\log (p)}{n}$  & $k_* \log (p)$ & \cite{liu2012multi} \\ 
    MCP &   ${\frac{k_*+\Delta_1\log(p)}{n}}$ & $k_* \log (p)$ &\cite{zhao2017pathwise} \\ 
     OMP &   $\frac{k_*+\Delta_1\log(p)}{n}$ & $k_*^2\log (p)$ &\cite{zhang2009consistency} \\ 
     FoBa & ${\frac{k_*+\Delta_1\log(p)}{n}}$    & $k_*\log (p)$ &\cite{zhang2011adaptive}  \\ \midrule

     Group Lasso&  $\frac{k_*+\bar{k}\log({m})}{n}$  & $k_*+\bar{k}\log({m})$ &\cite{huang2010benefit}   \\ 
Group OMP& $\frac{k_*\log(p)}{n}$ & $k_* \log (p)$ & \cite{ben2011near}\\ 
    IGA & ${\frac{k_*+ {\Delta_{\text{IGA}}}\log({m})}{{n}}}$    & $k_* + \bar{k} \log (m) $ & Our proposal  \\   \bottomrule
    \multicolumn{4}{l}{$\Delta_1:= { \left|\left\{i\in \text{supp}(\w^*):~|w^*_i| \leq \Omega\left(\sqrt{\log (p)\over n}\right)\right\}\right|}$} \\ 
\multicolumn{4}{l}{$\Delta_2:=k_*-\left|\left\{i:~i\text{-th largest element in}~|w_i^*| \geq \Omega{ \left(\sqrt{\frac{(k_*-i)\log (p)}{n}}\right)}\right\}\right|$} \\
\multicolumn{4}{l}{ $\Delta_{\text{IGA}}:=\left| { \left\{g\in G:~0<\|\w^*_g\| \leq \Omega\left(\sqrt{ {q+\log (m)\over n}}\right)\right\}}\right|$} \\ \bottomrule
    \end{tabular}
}
   
\end{table}

Different from these shrinkage-based methods, the other main algorithm
framework for group sparsity learning is the greedy-type
methods. Greedy algorithms have received much attention in literature. For example, one of the representative algorithms is a
forward greedy algorithm known as the orthogonal matching pursuit
(OMP, \citealp{mallat1993matching}). Its feature selection
consistency \citep{zhang2009consistency} and prediction/estimation
error bounds \citep{zhang2011sparse} have been established under an
irrepresentable condition and  a RIP condition, respectively.
Backward elimination steps were incorporated to greedy
algorithms in \citet{zhang2011adaptive}  to allow correction of
mistakes made by forward steps (known as FoBa algorithm). This seminal
work attained more refined estimation results than that of OMP and the Lasso by considering
the delicate scenario of varying group signal strengths. Similar
results on refined estimation error bound were also achieved by shrinkage-based methods with
non-convex penalties \citep{zhao2017pathwise}. Building on key
understandings of the aforementioned greedy algorithms, much efforts
were made to extend the OMP algorithms to handle group sparsity
\citep{swirszcz2009grouped, ben2011near, lozano2011group}. Although
these group OMP algorithms have shown promising empirical performance,
somewhat surprisingly, the explicitly justifiable benefits of group sparsity for greedy-type
algorithms in terms of estimation convergence rate were yet to be
rigorously established under a general setting. 

Specifically, the $\ell_2$ estimation error
bounds of some representative existing (group) sparsity learning
methods under linear models are
summarized  
in Table~\ref{table:comparison}, and we will describe the improved
convergence rate in the following
Section~\ref{subsec:contribution}. Here, 
${\rm supp}(\w)$ is the index set of nonzero elements in $\w$, and we define
$\norm{\w}_0$ to be the size of ${\rm supp}(\w)$. Then given the true
coefficient $\w^*$, $k_*=\norm{\w^*}_0$ represents the total
number of nonzero
elements in $\w^*$. Also define the group $\ell_0$ norm 
$\|\w\|_{G,0} =\sum_{g\in G}I(\norm{\w_g}>0)$ to be the number of
groups in $G$ that corresponds to nonzero coefficients. Given $\w^*$ and any group $g\in G$, if all elements in  $\mathbf w_g^*$
are zeros, then it is an irrelevant group to the response; otherwise, it is a
relevant group. Then,  $\bar k=\norm{\w^*}_{G,0}$ represents the total number
of relevant groups in the model. Let ${\Omega}\left(h_1(x)\right)$ be a function $h_2$ such that there exist two positive constants
$M$ and $N$ with $N\leqslant M$, and  we have $N |h_1(x)| \leqslant
|h_2(x)|\leqslant M |h_1(x)|$  for sufficiently large $x$. We will also use the notations defined here 
for the rest of this paper.

\subsection{Our Contribution}
\label{subsec:contribution}

The brief overview on related work above gives rise to three
intriguing questions: (1) Can we devise a greedy-type algorithm that
provides theoretically guaranteed benefits of group sparsity? Can we attain more refined estimation
error bound  than  the group Lasso? (2) If so, is there any theoretical guarantee
on correct support recovery of relevant groups? (3) Is our proposal
practically relevant to modern industrial applications and easy to use
for practitioners? 

\begin{itemize}[leftmargin=*]
\item  As the main contribution of our work, the proposed IGA
  algorithm affirmatively addresses the three questions above
  simultaneously. Consider a  group sparse linear model with
  even group size $q=p/m$ and define the squared $\ell_2$ estimation
  error $L_2(\hat\w):=\norm{\hat\w-\w^*}^2$. Under a
  variant of the RIP condition, the IGA  algorithm has
  $L_2(\hat\w)=O_p{ \left(\frac{k_*+ {\Delta_{\text{IGA}}}\log({m})}{{n}}\right)}$,
  where $\Delta_\text{IGA}$ is the number of nonzero groups with
  relatively weak signals (see Table~\ref{table:comparison} for
  precise definition). This estimation error has convergence rate
  matching   $O_p{\left(\frac{k_*+\bar{k}\log({m})}{n}\right)}$  under worst-case scenarios,  confirming that our proposed
  greedy-type algorithm indeed
  gains the benefits of group sparsity as opposed
  to the slower convergence rate of $O_p(k_*\log(p)/n)$ for many standard sparsity learning methods. The desirable
  group support recovery property is also established for our
  proposal. 


\item Our proposal further develops a group
  forward-backward stepwise strategy by introducing extra algorithm
  flexibility to widen its applicability. First, it is worth
  noting that IGA is proposed under a general M-estimation framework
  and can  apply to a broad class of criterion functions well beyond a
  standard linear model. In particular, under a GLM setting, we study
  the statistical properties of IGA for sparse logistic regression. We also
  demonstrate IGA's promising applications on
  sensor selection for human activity recognition under an extended setting.  Another interesting layer of
  flexibility  comes from an adjustable  human interaction parameter that
can potentially give practitioners more
  freedom to incorporate their own domain knowledge in feature
  selection. Moreover, to further improve computational efficiency, we
  propose and study an algorithm variant called Gradient-based IGA,
  which, besides desirable statistical properties, enjoys much faster computation than the
  original IGA.

\end{itemize}

The rest of the paper is organized as follows. We introduce the
proposed IGA algorithm, its heuristics, and the gradient-based IGA in
Section~\ref{sec:alg}. Section~\ref{sec:mainresult} provides the
theoretical analysis with a general standard criterion function
$Q(\w)$. Sections~\ref{sec:least} and \ref{subsec:logisticReg}
consider linear model and logistic
regression as important special cases and show explicit
estimation error bounds and group support
recovery properties. Sections~\ref{sec:exp} and \ref{sec:app} show empirical
studies using the IGA and its gradient-based variants in
comparison with other state-of-the-art feature/group selection
algorithms. All technical proofs and lemmas are left in the online Supplement.

\section{Our Proposal}\label{sec:alg}

The proposed IGA algorithm is an iterative stepwise greedy
algorithm. At each iteration, a forward group selection step is performed,
followed by a   backward group elimination  step.  The forward step aims to
identify one potentially useful group to be added to the
estimator. This step typically identifies a set of candidate groups
not selected yet,
that can drive down the criterion function $Q(\w)$ the most. The selected group is then
added to  the current group set for the next backward step. The backward step intends to  fix
mistakes made by  early steps and eliminate redundant groups that do not
significantly  drive down $Q(\w)$.  The foward-backward iteration
stops when no group can be selected in the forward step. Here, IGA is named
{\it interactive}  because in the forward step, rather than always choosing
the top one group to reduce the criteria
function most, we introduce an ``interactive'' discount parameter $\lambda$ that allows us to
consider the top few candidate
groups based on ranking and magnitude of criterion function reduction; a human operator is then allowed to select one
group of his/her choice from these candidate groups, possibly based on
domain knowledge  or experience.  In the following, we provide some
intuitions and heuristics of IGA algorithm using a simple numerical
example in Section~\ref{subsec:heuristics}, and then describe detailed
algorithmic procedures in Section~\ref{subsec:expalg}.

\subsection{A Heuristic Example}
\label{subsec:heuristics}

Different from a standard forward stepwise algorithm,
IGA has both  backward-elimination steps and an
``interactive'' discount parameter. To gain some intuitions and heuristics,
consider a simple example with
$m=5$ candidate groups, and each group $g$  ($g=1,\cdots,m$) contains 2 individual
feature variables $X_{g,1}$ and $X_{g,2}$.  Suppose that response $Y$
is generated by the first two groups as
\begin{equation*}
Y=(X_{1,1}+X_{1,2})+(X_{2,1}+X_{2,2})+N(0,1).
\end{equation*}
Let  $(X_{g,1},X_{g,2})$ all have independent standard normal distributions for
$g=1,2,4,5$, and define for $g=3$ that 
\begin{equation*}
X_{3,1} = X_{1,1}+X_{1,2}+N(0,0.5) \text{ and } X_{3,2} = X_{2,1}+X_{2,2}+N(0,0.5).
\end{equation*}
Then although groups 1 and 2 are the true
relevant groups, the forward selection will likely pick group 3
first as it can drive down the criterion function the most compared
to the relevant groups individually. Indeed,
we simulated $n=400$ data points and used the forward stepwise 
selection only (that is, always greedily add the group that gives
smallest $Q(\w)$), and as expected, the path of group selection was
\begin{equation}
\label{eq:forwardnaive}
\{3,\,    2,\,     1,\,     5,\,     4\},
\end{equation}
and, cross validation (CV) selected the first three groups
$\{3,1,2\}$ from the path. The group selection mistake made in the first step cannot be
corrected and resulted in overfitting. To correct this mistake, IGA
incorporates the backward
selection technique that re-visits the selected groups to eliminate
possibly redundant groups. Then with the
forward-backward steps, the path of group selection in the simulation
became
\begin{equation*}
\{3,\,     2,\,     1,\,    -3,\,     5,\,     4,\,     3\},
\end{equation*}
where the first step mistake was corrected in the fourth step by
eliminating group 3, and CV correctly selected groups $\{3,\,     1,\,
2,\,    -3\}=\{1,2\}$ from the path. 

The ``interactive'' parameter $\lambda$ ($0<\lambda\leq 1$; to be
defined in Section~\ref{subsec:expalg}) can also be potentially helpful as it
allows the forward selection step to incorporate human operator's
expert opinion. For example, assume an expert gives his/her
priority list $\mathcal A_I\subset G$. Then, in each forward
step $k$, we will rank
the criterion function values of each added group and consider an
enlarged set of promising groups $\mathcal A_\lambda$ (as opposed to always considering the single top-ranked
group), where the set size is related to $\lambda$. Then, if $\mathcal
A_I\cap\mathcal A_\lambda\neq \emptyset$, we can follow both expert
opinion $\mathcal
A_I$ and forward selection top set $\mathcal A_\lambda$
to pick the group from  $\mathcal
A_I\cap\mathcal A_\lambda$ with smallest criterion function value; if $\mathcal
A_I\cap\mathcal A_\lambda = \emptyset$, we simply pick the top group by
forward selection and ignore  $\mathcal A_I$. Specifically, in the numerical
example above, suppose expert priority list $\mathcal A_I=\{1\}$ correctly
includes group 1. Then with $\lambda=0.4$ (selected by CV), the forward stepwise algorithm gave
the  path
\begin{equation}
\label{eq:path_human}
\{1,\,     2,\,     5,\,     4,\,     3\},
\end{equation}
which correctly put $\{1,2\}$ in the solution path while pushing group
3 to later steps (as opposed to \eqref{eq:forwardnaive} that selected
group 3 in the first step). Very different
from the usual ``offset'' option in regression, not all groups in
expert opinion list  $\mathcal A_I$ has to be 
in the final selected model, and  $\mathcal A_I$
is not necessarily all correct; interestingly, if we changed the expert
opinion list to $\mathcal A_I=\{1,4\}$ in the simulation,
where group 4 was incorrectly included in expert opinion list, the
resulting path remained the same as \eqref{eq:path_human}, and CV gave the final selected
group set $\{1,2\}$. In Section~\ref{sec:mainresult}, we will show
in theory that, with backward elimination, incorporating human
expert opinions in our algorithm is safe  in the sense that
consistency properties still hold if $\lambda$ is not too small. We also
repeated this simple experiement and observed similar results 
as described above. More experiment settings and simulation results 
are given in Section~\ref{sec:exp}.

\subsection{IGA Algorithm}
\label{subsec:expalg}
With the intuitions above, we are ready to present the detailed
pseudo-code in Algorithm \ref{alg:general}. Here, given any set  $A$, let  $|A|$ be its
cardinality. Given a group $g\in G$, let $F_g\subset \{1,2,\cdots,
p\}$ be the feature index set corresponding to group $g$; given a group set
$S\subseteq G$, define  $F_S=\cup_{g\in S} F_g$, which transforms the
group index set $S$ to the corresponding feature index set.  With the feature index set $F_g=\{i_1,i_2,\cdots,i_{\abs{F_g}}\}\subset\{1,\cdots,p\}$,
 define $\mathbf{E}_{g}:=\left[\mathbf e_{i_{1}}, \mathbf
   e_{i_{2}},\ldots, \mathbf e_{i_{|F_g|}}\right]$ to be a  $p\times
 \abs{F_g}$ matrix, where $\mathbf e_{i}\in \mathbb{R}^p$ is the unit
 vector whose  $i$th element is 1 and all others are 0. 
 
\begin{algorithm} [h!]
   \caption{IGA algorithm.}
   \label{alg:general}
  \begin{algorithmic}[1]
    \REQUIRE{$\delta>0$, $0<\lambda\leq 1$}
    \ENSURE{$\w^{(k)}$}
    \STATE{$k=0,\w^{(k)}=0, G^{(0)}=\emptyset$}
    \WHILE{TRUE} 
    \IF{{$ \displaystyle{Q(\w^{(k)})-\min_{g\notin G^{(k)},\boldsymbol\alpha\in\real^{\abs{F_g}}}Q(\w^{(k)}+\mathbf{E}_{g} \boldsymbol\alpha)}<\delta$} }
	    \STATE{Break}
        \ENDIF
	\STATE{{  select an element as $g^{(k)}$ from the group set\\ $\mathcal A_{\lambda}:=\{ g \notin G^{(k)} \  | \  
	(Q(\w^{(k)})-\underset{\boldsymbol\alpha\in\real^{\abs{F_g}}}{\operatorname{min}}\ Q(\w^{(k)}+\mathbf{E}_{g}
        \boldsymbol\alpha))\geqslant \lambda \ (Q(\w)-\underset{\tilde
          g \notin G^{(k)},\,\boldsymbol\alpha\in\real^{\abs{F_{\tilde g}}}}{\operatorname{min}}\ 
	\ Q(\w^{(k)}+\mathbf{E}_{\tilde g} \boldsymbol\alpha))\}$ } }
      \STATE{$G^{(k+1)}={G^{(k)}\cup \{g^{(k)}\}}$}
	\STATE{$\w^{(k+1)}=\underset{\substack{\w \in \mathbb{R}^p:\\
              \textrm{supp }(\w)\subset F_{G^{(k+1)}} }}{\operatorname{argmin}}\ Q(\w)$}
        \STATE{$\delta^{(k+1)}=Q(\w^{(k)})-Q(\w^{(k+1)})$}
	\STATE{$k=k+1$}
        \WHILE{TRUE}
	\IF{$\underset{g\in G^{(k)}}{\min}Q(\w^{(k)} -\mathbf{E}_{g} \w_g^{(k)})-Q(\w^{(k)})\geqslant\frac{\delta^{(k)}}{2}$}
             \STATE{Break}
	 \ENDIF
	 \STATE{$g^{(k)}=\underset{g\in G^{(k)}}{\operatorname{argmin}} \, Q(\w^{(k)} -\mathbf{E}_{g} \w_g^{(k)})$}
	  \STATE{$k=k-1$}
	  \STATE{$G^{(k)}=G^{(k+1)} {\setminus} \{g^{(k+1)}\}$}
	           	  	  \STATE{$\w^{(k)}=\underset{\substack{\w
                                        \in \mathbb{R}^p:\\
                                        \textrm{supp }(\w)\subset
                                        F_{G^{(k)}}} }{\operatorname{argmin}}\ Q(\w)$}
        	\ENDWHILE
    \ENDWHILE
  \end{algorithmic}
\end{algorithm}

Specifically, we initialize the current set of groups $G^{(k)}$ as an
empty set $G^{(k)}=\emptyset$ in Line 1, and then iteratively select
and delete feature groups from $G^{(k)}$.  Lines 3-5 provide the
termination test: IGA terminates when no  groups outside
$G^{(k)}$ can decrease the criterion function $Q(\cdot)$ by a fixed
threshold  $\delta>0$. 
Line 6 is the key step for  the forward group selection. First, we evaluate the quality of all  groups outside of the current group set $G^{(k)}$ and, given a discount factor $\lambda \in (0, 1]$, construct a  candidate group set $\mathcal A_{\lambda}$ that includes all ``good"  groups.
The human operator can then decide which group to select from
$\mathcal A_\lambda$. 
Here the discount factor $\lambda$ determines the extent to which
human interaction is desired.  Bigger $\lambda$  typically gives
smaller candidate group set and thereby less human involvement is allowed in group selection.
In Lines 7-8, we  recalculate the optimal coefficient of $\w$ supported on the updated group set. At the end of the forward step, we calculate the gain from this step as $\delta^{(k)}$ in Line 9, which will feed into the subsequent backward step.

In the backward step, we intend to check if any group in current group
set $G^{(k)}$ becomes redundant considering that some previously
selected  groups may be less important after new groups join
from the forward steps. In Lines 12-14, for each  group in
$G^{(k)}$, we calculate  the difference between the criterion function
value of the current  group set and the function
value with one group coefficients removed from the
current pool. If the smallest difference is less than the threshold
$\delta^{(k)}\over{2}$, in Lines 15-18, we  remove the least
significant group, update the current group set $G^{(k)}$ and
re-calculate current estimator $\w^{(k)}$. This group elimination process
continues until the function difference is less than
$\delta^{(k)}\over{2}$, i.e., all remaining selected groups are considered to have significant contribution.

\begin{algorithm}
   \caption{Gradient-based IGA algorithm.}
   \label{alg:gdt}
  \begin{algorithmic}[1]
    \REQUIRE{$\delta>0$, $0<\lambda\leq 1$}
    \ENSURE{$\w^{(k)}$}
    \STATE{$k=0,\w^{(k)}=0,F^{(0)}=\emptyset,G^{(0)}=\emptyset$}
    \WHILE{TRUE} 
    \IF{ { $\|\nabla Q(\w^{(k)})\|_{G,\infty}< \varepsilon$ }}
	    \STATE{Break}
        \ENDIF
	\STATE{  select any $g^{(k)} \textrm{ in } \{ g \notin G^{(k)}\ \ | 
	\ \ \|\nabla_{g} Q(\w^{(k)})\| \geqslant \ \lambda\ \underset{\tilde
          g \notin G^{(k)}}{\operatorname{max}}\|\nabla_{\tilde g} Q(\w^{(k)})\| \}$ }
      \STATE{$G^{(k+1)}={G^{(k)}\cup \{g^{(k)}\}}$}
      	         	\STATE{$\w^{(k+1)}=\underset{\substack{\w \in
                              \mathbb{R}^p:\\ \textrm{supp
                              }(\w)\subset F_{G^{(k)}}} }{\operatorname{argmin}}\ Q(\w)$}
        \STATE{$\delta^{(k+1)}=Q(\w^{(k)})-Q(\w^{(k+1)})$}
	\STATE{$k=k+1$}

        \WHILE{TRUE}
	\IF{$\underset{g\in G^{(k)}}{\min}Q(\w^{(k)} -\mathbf{E}_{g} \w_g^{(k)})-Q(\w^{(k)})\geqslant\frac{\delta^{(k)}}{2}$}
             \STATE{Break}
	 \ENDIF
	 \STATE{$g^{(k)}=\underset{g\in G^{(k)}}{\operatorname{argmin}}\ Q(\w^{(k)} -\mathbf{E}_{ g}\w_g^{(k)})$}
	  \STATE{$k=k-1$}
	  \STATE{$G^{(k)}=G^{(k+1)} {\setminus} \{g^{(k+1)}\}$}
	            	  	  \STATE{$\w^{(k)}=\underset{\substack{\w
                                        \in \mathbb{R}^p:\\
                                        \textrm{supp }(\w)\subset
                                        F_{G^{(k)}}} }{\operatorname{argmin}}\ Q(\w)$}
	\ENDWHILE
    \ENDWHILE
  \end{algorithmic}
\end{algorithm}

\subsection{Gradient-Based IGA Algorithm}
\label{subsec:giga}

Each forward
step in Algorithm~\ref{alg:general} requires to repeatedly perform optimization of the criterion
functions (Lines 3 and Line 6) to evaluate quality across all candidate
feature groups. It is interesting to note that these 
time-determining steps can be potentially replaced by
using gradients of the criterion function $\nabla Q(\w)$ so that we
can avoid the repeated function optimization tasks and perform
computation in a much more efficient way.

Specifically,  given a threshold
value $\varepsilon$ and  a discount factor $0<\lambda\leq 1$, we can
replace corresponding statements in Lines 3 and Line 6 with
$\|\nabla Q(\w^{(k)})\|_{G,\infty}< \varepsilon$  and  $\{ g \notin G^{(k)}
\ | \ \|\nabla_{g} Q(\w^{(k)})\| \geqslant \lambda \ \underset{\tilde g \notin
  G^{(k)}}{\operatorname{argmax}}\|\nabla_{\tilde g} Q(\w^{(k)})\| \}$,
respectively, where $\nabla_g$ denotes the gradient with respective to
group $g$, and  $\|\w\|_{G,\infty}:=\max_{g \in G} \|\w_{g}\|$ 
represents the $\ell_{2,\infty}$ norm for any 
$\w\in\real^p$.  Then we have a gradient-based version of the IGA algorithm
and we call this variant Gradient-based IGA algorithm (or GIGA  for
brevity). The detailed procedures of GIGA algorithm are given in
Algorithm~\ref{alg:gdt}. We will provide the theoretical properties of the GIGA algorithm in
Section~\ref{sec:mainresult} and compare its computational performance
with the original IGA algorithm in Section~\ref{sec:exp}.

\section{Theoretical Results}
\label{sec:mainresult}

This section provides theoretical analysis of
Algorithm~\ref{alg:general} and Algorithm~\ref{alg:gdt}. Using $\bar\w$ in \eqref{eq:groupl0} as a
useful benchmark, we consider performance of IGA and GIGA estimator $\w^{(k)}$
under a general criterion function $Q$ in
Section~\ref{subsec:results_general}.  Building on key results from this
general setting in Section~\ref{subsec:results_general}, we study sparse linear models in
Section~\ref{sec:least} and sparse logistic regression in
Section~\ref{subsec:logisticReg} to provide valuable insights into coefficient estimation and group support recovery.

\subsection{A General Setting}
\label{subsec:results_general}
We first introduce relevant definitions and assumptions regarding the
setting on the general criterion function $Q(\cdot)$ in \eqref{eq:groupl0}. Unless stated
otherwise, we consider fixed designs in this section. Let  $\bar{G}$ be the sparsest
group set covering all nonzero elements in $\bar{\w}$ and let
$\mathcal D\subset \real^p$ be a compact feasible region for $\w$. Given vectors $\mathbf u, \mathbf v\in\real^p$,
define $\langle\mathbf u, \mathbf v \rangle=\mathbf u^T\mathbf v$ to
be the inner product. Given any $F\subset \{1,\cdots,p\}$, define 
\begin{align}
\rho_{-}(F):= & \inf_{\operatorname{supp}(\u)\subset F, \w \in \mathcal{D}} \frac{Q(\w+\u)-Q(\w)-\langle \nabla Q(\w), \u\rangle }{{1\over 2}\|\u\|^{2}},
\label{eq:def_rho-}
\\
\rho_{+}(F):= & \sup_{\operatorname{supp}(\u)\subset F, \w\in \mathcal{D}} \frac{Q(\w+\u)-Q(\w)-\langle \nabla Q(\w), \u\rangle }{{\frac{1}{2}}\|\u\|^{2}}.
\label{eq:def_rho+}
\end{align}
The restricted strongly convexity parameter and the restricted Lipschitz smoothness parameter \citep{huang2010benefit} are 

\begin{align}
\varphi_{-}(t)= & \operatorname{inf}\{\rho_{-}(F_S)~|~S\subset G,{|S|}\leqslant t\},
\label{eq:def_phi-}
\\
\varphi_{+}(t)= & \operatorname{sup}\{\rho_{+}(F_S)~|~S\subset G,{|S|}\leqslant t\},
\label{eq:def_phi+}
\end{align}
respectively. Define the restricted condition number as $\kappa(t) :=\frac{\varphi_{+}(t)}{\varphi_{-}(t)}$. 
\begin{assumption} \label{ass:s}
There exists a positive integer $s$ satisfying
\begin{align}
&\Omega(1)\leq \varphi_-(s)\leq \varphi_+(s) \leq  \Omega(1), \label{ass:constant1} \\
&s>\bar{k}+\frac{(\bar{k}+1)(\sqrt{\kappa(s)}+1)^2(\sqrt{2}\kappa(s))^2}{\lambda}.  \label{ass:sparsity}
\end{align}
\end{assumption}

In Assumption~\ref{ass:s}, \eqref{ass:constant1} requires that
$\varphi_+(s)$ and $\varphi_-(s)$ are upper bounded and  bounded away
from zero for large enough $n$. Given least square loss  $f_i(z) =
{1\over 2}\|z-y_i\|^2$, the criterion function is
$Q(\w)={\frac{1}{2n}}\norm{\y-X\w}^2$ with $\y\in\real^n$ being the
response vector and $X$ being the $n\times p$ design matrix; then
\eqref{ass:constant1} becomes a variant of the group-RIP condition
\citep{huang2010benefit}. Indeed, \eqref{ass:constant1} is satisfied
if there is a constant  $0<\eta< 1$, such that 
\begin{equation}
\label{eq:eigenval} 
(1-\eta)\|{\w}\|^2\leqslant \frac{1}{{n}}\|X\w\|^2 \leqslant
(1+\eta)\|\w\|^2
\end{equation}
for all $\|\w\|_{G,0}\leqslant s$. Similarly, under logistic regression
with $f_i(z) = \log (1+\exp(-y_iz))$ ( $y_i\in \{-1, 1\}$) and bounded
covariates in $l_2$ norm,  \eqref{ass:constant1} is also satisfied
with \eqref{eq:eigenval}.  The requirement in \eqref{ass:sparsity} of Assumption~\ref{ass:s} is
also mild. Note that  $\kappa(s)$ is  bounded if  $Q(\w)$ is a
strongly convex function with bounded Lipschitzian gradient. 

As mentioned, we assume the design matrix $X$ in the standard framework \eqref{eq:groupl0} is deterministic
throughout the section. However, it is also noted that our results can be readily extended
to random designs that satisfy Assumption~\ref{ass:s} with high
probability. For example, Lemma~\ref{lem:cond}  in Supplement~\ref{sec:suppthm} shows that
Assumption~\ref{ass:s} holds with high probability if   $\x_i$'s are sub-Gaussian and  the condition number of
$\nabla^2 f_i(\cdot)$ is bounded; least square function naturally
satisfies this requirement; with bounded covariates, logistic regression
also falls under this scenario.

\begin{theorem} 
\label{thm:main}
Suppose Assumption~\ref{ass:s} holds and $\delta= C_\lambda\|\nabla
Q(\bar{\w})\|_{G,\infty}^{2}$, where
$C_\lambda=\frac{8\varphi_{+}(1)}{\lambda\varphi_{-}^{2}(s)}$. Then
IGA estimator $\w^{(k)}$ has the following properties:

\begin{itemize}
\item Algorithm~\ref{alg:general} terminates at $k < s -\bar{k}$;
\item $
\|\bar{\w}-\w^{(k)}\|^2 \leqslant \Omega\left(\lambda^{-1}\|\nabla Q(\bar{\w})\|_{G, \infty}^2 \bar{\Delta} \right)$;
\item $Q(\w^{(k)})-Q(\bar{\w})\leqslant \Omega\left(\lambda^{-1}\|\nabla Q(\bar{\w})\|_{G, \infty}^2 \bar{\Delta} \right)$;
\item $|G^{(k)} \setminus \bar{G} | + | \bar{G}\setminus G^{(k)} | \leqslant \Omega\left({\lambda^{-1}} \bar{\Delta}\right)
$,
\end{itemize}
where $\bar{\Delta} :=\left|\{g\in \bar{G} :\| \bar{\w}_{g}\|<\Omega (\lambda^{-1}\|\nabla Q(\bar{\w})\|_{G, \infty})\}\right|$.
\end{theorem}

The first claim states that the algorithm terminates.  Using $\bar\w$
as a benchmark, the second and third bounds  provide the  $\ell_2$
distance and the criterion function discrepancy of IGA estimator
$\w^{(k)}$, respectively. The last bound describes the group feature
selection difference between $\w^{(k)}$ and $\bar\w$. Note that all
error bounds depend on $\bar{\Delta}$, which  counts the number of
``weak" signal groups among  nonzero groups in $\bar{\w}$. If all nonzero
groups are strong enough and $\bar{\Delta}$ turns out to be zero, then
$\w^{(k)}$ becomes  equivalent to $\bar{\w}$ and the original ``idealized''
problem~\eqref{eq:groupl0} is exactly solved by
Algorithm~\ref{alg:general}. Also, although the discount
factor $\lambda$  slightly affects the error bound, if  $\lambda$ is
bounded away from zero, it would not change the order of error bounds
above while allowing the extra flexibility in human expert participation. 

Analysis of the gradient-based GIGA in Algorithm~\ref{alg:gdt} provides
theoretical results parallel to those for plain vanilla IGA in Algorithm
\ref{alg:general}. Indeed, as
summarized in Theorem~\ref{thm:main_giga}, with fixed $\lambda$, the
rates of coefficient estimation and the group feature selection error
bounds remain the same for GIGA as that of Theorem~\ref{thm:main} for IGA. 

\begin{theorem} 
\label{thm:main_giga}
Suppose Assumption~\ref{ass:s} holds and $\varepsilon= \tilde C_\lambda\|\nabla
Q(\bar{\w})\|_{G,\infty}$, where
$\tilde C_\lambda=\frac{2\sqrt{2}\varphi_{+}(1)}{\lambda\varphi_{-}(s)}$. Then
GIGA estimator $\w^{(k)}$ has the following properties:

\begin{itemize}
\item Algorithm~\ref{alg:gdt} terminates at $k < s -\bar{k}$;
\item $
\|\bar{\w}-\w^{(k)}\|^2 \leqslant \Omega\left(\lambda^{-2}\|\nabla Q(\bar{\w})\|_{G, \infty}^2 \bar{\Delta} \right)$;
\item $Q(\w^{(k)})-Q(\bar{\w})\leqslant \Omega\left(\lambda^{-2}\|\nabla Q(\bar{\w})\|_{G, \infty}^2 \bar{\Delta} \right)$;
\item $|G^{(k)} \setminus \bar{G} | + | \bar{G}\setminus G^{(k)} | \leqslant \Omega\left({\lambda^{-2}} \bar{\Delta}\right)
$,
\end{itemize}
where $\bar{\Delta}=\left|\{g\in \bar{G} :\| \bar{\w}_{g}\|<\Omega (\lambda^{-1}\|\nabla Q(\bar{\w})\|_{G, \infty})\}\right|$.
\end{theorem}

Using the general results obtained in
Theorem~\ref{thm:main} and Theorem~\ref{thm:main_giga}, we can demonstrate statistical properties of the
IGA estimators under two special and important statistical model scenarios, which include sparse
linear model in Section~\ref{sec:least} and sparse logistic regression
in Section~\ref{subsec:logisticReg}. One interesting quantity from the
theorems above is the gradient $\nabla Q(\bar\w)$, and
by construction,  $\nabla_{\bar G} Q(\bar\w)=\mathbf 0$. In the following, we assume $\bar\w$
is the unique solution of $\nabla_{\bar G} Q(\w)=\mathbf 0$. We will focus on  the cases that all $m$ groups have equal  size $q=p/m$, although our analysis can allow arbitrary group sizes.  

\subsection{Sparse Linear Model}
\label{sec:least}
 Consider true model that response vector $\y\in\real^n$ is generated
 from $X\w^* +\bd\varepsilon$ with $\bd\varepsilon\sim N(\mathbf 0_n,I_n)$.   We use standard
 normal errors here, but our results can be easily generalized to
 sub-Gaussian errors. Suppose the columns of design matrix $X$ are
 normalized to $\sqrt{n}$ in $l_2$ norm.  We consider analysis of IGA algorithm with
 the least square criterion function
 $Q(\w)={\frac{1}{2n}}\norm{\y-X\w}^2$.  With  $\nabla Q(\w)=\frac{1}{n}X^T(X\w-\y)$, the following
 Theorem~\ref{probability_hypo_sp} connects the gradients in
 Theorem~\ref{thm:main} and Theorem~\ref{thm:main_giga} with
 the sparse linear model.

\begin{theorem}
\label{probability_hypo_sp}
 Suppose Assumption~\ref{ass:s} holds. Then 
 the true coefficient vector  $\w^{\ast}$ and the benchmark estimator $\bar\w$ satisfies
\begin{equation*}
  \|\nabla Q(\w^{\ast})\|_{G,\infty}\leqslant O_p{\left(\sqrt{\frac{q+\log
    m}{n}} \right)} \quad\text{and}\quad \|\nabla Q(\bar\w)\|_{G,\infty}\leqslant O_p {\left(\sqrt{\frac{q+\log
    m}{n}} \right)}.
\end{equation*}
In addition, we have  $ \max_{g\in\bar G}\norm{\bar\w_g-\w_g^*} \leq O_p\Bigl(\sqrt{\frac{q+\log m}{n}}\Bigr)$.
\end{theorem}
Consequently, by combining Theorem~\ref{thm:main} (or Theorem~\ref{thm:main_giga}) and
Theorem~\ref{probability_hypo_sp}, we can obtain explicit IGA (or GIGA)
estimation upper bounds for sparse linear model. In particular, both coefficient estimation  and group support
recovery can be shown in the following Theorem~\ref{thm:linear}.
Recall from the definition in Table~\ref{table:comparison} that 
$\Delta_{\text{IGA}}$ is the cardinality of the set of groups with relatively
weak signals. 

\begin{theorem} 
\label{thm:linear}
Suppose conditions of Theorem~\ref{thm:main} (or Theorem~\ref{thm:main_giga}) hold and  $\lambda$ is
bounded away from zero. Then IGA (or GIGA) estimator $\w^{(k)}$  for sparse
linear model has following statistical properties:
\noindent
\begin{itemize}[leftmargin=0cm,itemindent=.5cm,labelwidth=\itemindent,labelsep=0cm,align=left]
\item $L_2(\w^{(k)}):=\|\w^{(k)} - \w^*\|^2 = O_p \Bigl( {k_*+\Delta_{\text{IGA}}\log m\over n}\Bigr).$
\item There is a constant $C>0$ (not depending on $n$) such
  that if  $\min_{g\in \bar G}\norm{\w_g^*}\geq  C
  \sqrt{\frac{q+\log(m)}{n}}$, then group selection
  consistency holds that 
$P(G^{(k)}=\bar G)\rightarrow 1$ as $n\rightarrow \infty$.
\end{itemize}
\end{theorem}

Interestingly, the estimation consistency in Theorem~\ref{thm:linear}
indeed demonstrates the benefits of group sparsity in estimation convergence rate: the estimation
error $L_2(\w^{(k)})$ is upper bounded by $O_p{\left({k_*+\bar k\log m\over
  n}\right)}$, which improves over the rate $O_p{\left(\frac{k_*\log(p)}{n}\right)}$  of
standard sparsity. In addition, our results can be  more refined than
the group Lasso in the existence of relatively strong group signals. In
particular, under the beta-min condition that coefficients of all $\bar k$ relevant groups have
$\ell_2$-norm lower bounded by  $\Omega{\left(\sqrt{ {q+\log m\over n}}\right)}$,
$L_2(\w^{(k)})$ is improved to $O_p(k_*/n)$ by removing an additive
term $\bar k\log m/n$, which can be a substantial improvement in high-dimensional
settings $m\gg n$ with relatively strong group signals; under the same conditions, Theorem~\ref{thm:linear} also shows
that the IGA (or GIGA) estimator is consistent in group support
recovery.

\subsection{Sparse Logistic Regression}
\label{subsec:logisticReg}

We assume the sparse logistic regression setting here with binary response
$y_i\in\{-1,\,1\}$ and
$\P(y_i=1\given\x_i)=\frac{\exp(\x_i^T\w^*)}{1+\exp(\x_i^T\w^*)}$  for $1\leq i\leq n$.  Then
with the negative log-likelihood criterion function
$Q(\w)=\frac{1}{n}\sum_{i=1}^n\log(1+\exp(-y_i\x_i^T\w))$, the
gradient is $\nabla Q(\w)=\frac{1}{n}X^T\h(\w)$, where
$\h(\w)=(h_1(\w),\cdots,h_n(\w))^T$ and
$h_i(\w)=\frac{-y_i}{1+\exp(y_i\x_i^T\w)}$.  Let $X_g$ and $X_{\bar G}$
 be the design
matrix corresponding to group $g$ ($1\leq g\leq m$) and group set  $\bar G$, respectively. Define $n\times n$ diagonal matrix
$W(\w)=\diag(\nu_1(\w),\cdots,\nu_n(\w))$, where
$\nu_i(\w)=p_i(\w)(1-p_i(\w))$ and
$p_i(\w)=\frac{\exp(\x_i^T\w)}{1+\exp(\x_i^T\w)}$. Set
$W^*=W(\w^*)$. We 
then consider the 
quantities
$U_1=\norm{\frac{1}{n}X_{\bar G}^T X_{\bar G}}$,
$U_2=\norm{(\frac{1}{n}X_{\bar G}^TW^*X_{\bar G})^{-1}}$ and
$U_3=\max_{g\in G\backslash \bar G}\norm{\frac{1}{n}X_g^TW^*X_{\bar
    G}}$. Variants
of these quantities have been used to study properties of 
shrinkage-type approaches (e.g., \citealp{fan2014strong}). 
We now connect the gradients in
Theorem~\ref{thm:main} and Theorem~\ref{thm:main_giga} with the sparse logistic regression through the
following Theorem~\ref{thm:logistic_oracle}, which gives 
results similar to that of Theorem~\ref{probability_hypo_sp}.

\begin{theorem}
\label{thm:logistic_oracle}
 Suppose Assumption~\ref{ass:s} holds. Assume $\bar k^2(q+\log
 m)=o(n)$ and quantities $U_1$, $U_2$,
 $U_3$ are upper bounded. Then we have 
\begin{equation*}
  \|\nabla Q(\w^{\ast})\|_{G,\infty}\leqslant O_p{\left(\sqrt{\frac{q+\log
    m}{n}} \right)}, \quad \|\nabla Q(\bar\w)\|_{G,\infty}\leqslant O_p{\left(\sqrt{\frac{q+\log
    m}{n}} \right)},
\end{equation*}
and  $ \max_{g\in\bar G}\norm{\bar\w_g-\w_g^*} \leq O_p\Bigl(\sqrt{\frac{q+\log m}{n}}\Bigr)$.
\end{theorem}

Combining Theorem~\ref{thm:main} (or Theorem~\ref{thm:main_giga}) and
Theorem~\ref{thm:logistic_oracle}, we establish the IGA (or GIGA) estimator's
statistical properties in
Theorem~\ref{thm:logistic} for logistic regression. The conclusion similar to that of sparse linear model still holds here,
which provides a refined estimation
convergence rate. In particular, if all the
relevant groups have relatively
large signals, the coefficient estimation error can be improved to
$O_p(k_*/n)$ and is consistent in group support recovery.

\begin{theorem} \label{thm:logistic}
Suppose conditions of Theorem~\ref{thm:main} (or Theorem~\ref{thm:main_giga}) and Theorem~\ref{thm:logistic_oracle} hold and  $\lambda$ is
bounded away from zero. Then IGA (or GIGA) estimator $\w^{(k)}$  under the logistic
regression setting has following estimation and group selection properties:
\noindent
\begin{itemize}[leftmargin=0cm,itemindent=.5cm,labelwidth=\itemindent,labelsep=0cm,align=left]
\item $\|\w^{(k)}-\w^*\|^2 = O_p \Bigl( {k_*+\Delta_{\text{IGA}}\log m\over n}\Bigr).$
\item There is a constant $\tilde C>0$ (not depending on $n$) such
  that if  $\min_{g\in \bar G}\norm{\w_g^*}\geq \tilde C
  \sqrt{\frac{q+\log(m)}{n}}$, then 
$P(G^{(k)}=\bar G)\rightarrow 1$ as $n\rightarrow \infty$.
\end{itemize}
\end{theorem}

\section{Simulation}\label{sec:exp}

In this section, we evaluate the performance of the proposed algorithms on
simulation data. As illustrated in the heuristic example of
Section~\ref{subsec:heuristics}, the forward selection and backward elimination scheme
in IGA naturally provides a group selection
path similar to the solution path of shrinkage-based methods like the
group Lasso; correspondingly, rather than directly setting the parameter $\delta$ to
determine the final model's group
sparsity level, we can equivalently generate a group selection path 
and choose the appropriate sparsity level $k$ from the
path to estimate the final model. Without setting constraints on model's
sparsity level, we used ten-fold CV to
automatically determine $k$ in the simulation.

Another parameter in IGA is the ``interactive'' parameter
$\lambda$, which potentially allows help from human expert opinions. As
illustrated in Section~\ref{subsec:heuristics}, if an expert provides a
priority list $\mathcal A_I$, and $\mathcal
A_I\cap\mathcal A_\lambda \neq \emptyset$ in the forward selection step, then
IGA adds the group in $A_I\cap\mathcal A_\lambda$ that gives the
smallest $Q(\w)$; if $\mathcal
A_I\cap\mathcal A_\lambda = \emptyset$ in the forward selection step,
IGA adds the group in $\mathcal A_\lambda$  that gives the
smallest $Q(\w)$; if there is
no expert opinion (or $\mathcal A_I=\emptyset$), simply set
$\lambda=1$. Recall from Table~\ref{table:comparison} that  $\bar k=\norm{\w^*}_{G,0}$ is the total number
of relevant groups in the true model. To mimic a more realistic scenario that expert
opinion may contain both correct and incorrect components,  $\mathcal A_I$ correctly contains
$\lfloor \frac{3}{5} \bar k \rfloor$ of relevant groups and incorrectly
contains an equal number of $\lfloor \frac{3}{5} \bar k \rfloor$
irrelevant groups throughout this simulation study. With $\mathcal A_I$,  we used ten-fold CV  to automatically
select $(k, \lambda)$, and $\lambda$'s candidate values were 
$\{0.2, 0.4, 0.6, 0.8, 1.0\}$. To differentiate methods based on whether IGA 
uses $\mathcal A_I$, we denote the IGA algorithm that selected
$\lambda$ with CV  as  ``IGA-$\lambda$'', and denote the IGA algorithm
with pre-specified 
$\lambda=1$ (that is, it ignored the interactive parameter and $\mathcal A_I$) as ``IGA''.
 
In light of our previous theoretical understandings, to gauge the proposal's
numerical performance, we considered FoBa and the group Lasso as the
representative benchmark methods and implemented them in MATLAB. Both methods
have their own tuning parameters: the sparsity-level parameter of
FoBa was tuned the same way as $k$ in IGA, but the known group structure
was ignored; we also implemented the group Lasso by the accelerated
proximal gradient descent \citep{beck2009fast} and used ``warm start'' to generate the
solution path (that corresponds to a decreasing sequence of shrinkage tuning
parameters; \citealp{friedman2010regularization}). Ten-fold CV was used to select tuning
parameters for these benchmark methods. 

Let $\hat\w$ be the estimator of an
algorithm and let  $\hat G$  be the set of nonzero groups in
$\hat\w$. To compare the coefficient estimation performance, we
considered the estimation error $\norm{\hat\w-\w^*}$. To evaluate the
group support recovery performance, we used the number of
correctly identified relevant groups $\abs{\hat G\cap \bar G}$ and the
number of incorrectly identified relevant groups $\abs{\hat
  G\backslash \bar G}$.

\begin{table}[!ht]
\caption{Averaged simulation results for Case 1 based on 100 runs.}
\begin{center}
\scalebox{0.9}{
\begin{tabular}{lcccccccccc}
\toprule 
& \multicolumn{5}{c}{$\beta=0.4$} & \multicolumn{5}{c}{$\beta=1$} \tabularnewline
\cmidrule(lr){2-6} \cmidrule(lr){7-11}
\multicolumn{1}{c}{$\bar k$} & 5 & 7 & 9 & 11 & 13 & 5 & 7 & 9 & 11 & 13 \tabularnewline
\midrule
 \multicolumn{2}{l}{$\norm{\hat\w-\w^*}$}   &  &  &  &  &  &  &  &  &  \tabularnewline
\cmidrule(lr){1-2}
FoBa & 2.01  & 2.39 & 2.75 & 3.06  & 3.36
  & 2.59  & 4.55  & 5.89  & 6.83 & 7.64\tabularnewline
&  (0.01)  & 
 (0.01) & (0.01) & (0.01)  &  (0.01)
  & (0.11)  & (0.08)  & (0.04)  & (0.04) & (0.04)\tabularnewline
group Lasso & 1.43  & 1.65  & 1.84 & 2.06 
  & 2.21  & 1.76  & 2.09  & 2.37  & 2.70
  & 2.96 \tabularnewline
 & (0.03)  & (0.02)  & (0.02) & (0.02)
  & (0.02)  & (0.03)  & (0.03)  & (0.03)  & (0.03)
  & (0.04)  \tabularnewline
IGA & 0.97 & 1.14  & 1.32 & 1.45  & 1.62
  & 1.15  & 1.30  & 1.42 & 1.55  & 1.66 \tabularnewline
 & (0.02)& (0.02)  & (0.03)  & (0.03)  & (0.03)
  & (0.01)  & (0.01)  & (0.01)  & (0.02)  & (0.02)  \tabularnewline
IGA-$\lambda$ & 0.92 & 1.11 & 1.25  & 1.39
  & 1.57 & 1.08  & 1.19  & 1.32  & 1.44
  & 1.60  \tabularnewline
 & (0.02)  & (0.02) & (0.02)  & (0.02)
  & (0.02)  & (0.01)  & (0.01)  & (0.01)  &  (0.02)
  & (0.02)  \tabularnewline
GIGA & 1.04 & 1.34 & 1.55  & 1.84
  & 2.05 & 0.99  & 1.14  & 1.29  & 1.38
  & 1.59  \tabularnewline
 & (0.02)  & (0.03) & (0.03)  & (0.04)
  & (0.04)  & (0.01)  & (0.01)  & (0.01)  &  (0.02)
  & { (}0.03)  \vspace{0.05in}\tabularnewline

\multicolumn{2}{l}{$\abs{\hat G\cap \bar G}$}   &  &  &  &  &  &  &  &  &  \tabularnewline
\cmidrule(lr){1-2}
FoBa & 3.59 & 4.72   & 5.61 & 6.64  & 7.27  & 4.99  & 6.77  & 8.51  &
                                                                      9.73  & 11.13  \tabularnewline
group Lasso & 4.39  & 6.43  & 8.50  & 10.28 & 12.25  & 5.00  & 7.00  &
  9.00 & 11.00  & 13.00 \tabularnewline
IGA & 4.77   & 6.68  & 8.59  & 10.45  & 12.26  & 5.00  & 7.00  & 9.00
  & 11.00  & 13.00 \tabularnewline
IGA-$\lambda$ & 4.89 & 6.83  & 7.56 & 10.70  & 12.48  & 5.00  & 7.00
  &9.00  & 11.00  & 13.00 \tabularnewline
GIGA & 4.53 & 5.90  & 8.79 & 8.69  & 10.13  & 5.00  & 7.00
  &9.00  & 11.00  & { 12.94} \vspace{0.05in}\tabularnewline

 \multicolumn{2}{l}{$\abs{\hat G\backslash \bar G}$}   &  &  &  &  &  &  &  &  &  \tabularnewline
\cmidrule(lr){1-2}
FoBa & 1.19  & 1.29  &1.38 & 2.16  & 1.93 & 5.34 & 4.26  & 3.77  & 13.86  & 4.78 \tabularnewline
group Lasso & 5.12  & 10.01  & 14.27  & 17.28  & 21.24  & 11.99  &
                                                                   21.36
  & 27.47  & 17.20  & 39.03  \tabularnewline
IGA & 0.70  & 0.66  & 0.72  & 0.67  & 0.83  & 1.97  & 1.99 & 1.92  &
                                                                     2.22  & 1.95  \tabularnewline
IGA-$\lambda$ & 0.79 & 0.98  & 0.82  & 0.78  & 0.82  & 2.00  & 2.04  &
                                                                    2.00
  & 4.42  & 1.98  \tabularnewline
GIGA & 0.98 & 1.07  & 1.15  & 1.27  & 1.11  & 1.99  & 2.00  &
                                                                    1.99
  & 1.96  & 1.85  \tabularnewline

\bottomrule
\end{tabular}
}
\end{center}
\label{tab:sim1}
\end{table}

\begin{figure}[!ht]
\vspace{-.1in}
{
\centering
\includegraphics[scale=.65]{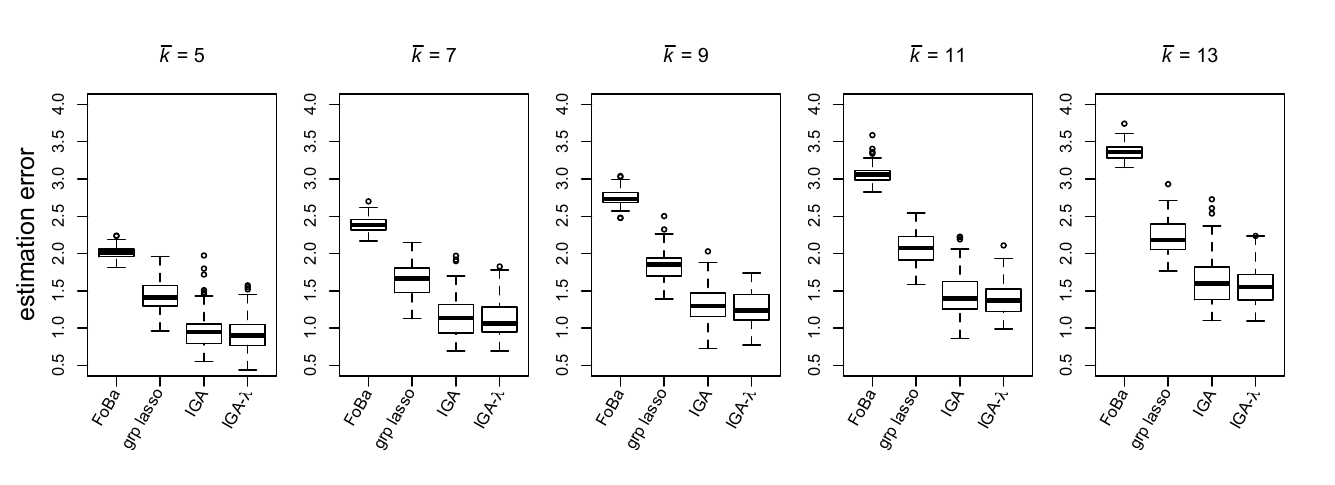}
\vspace{0in}
\includegraphics[scale=.65]{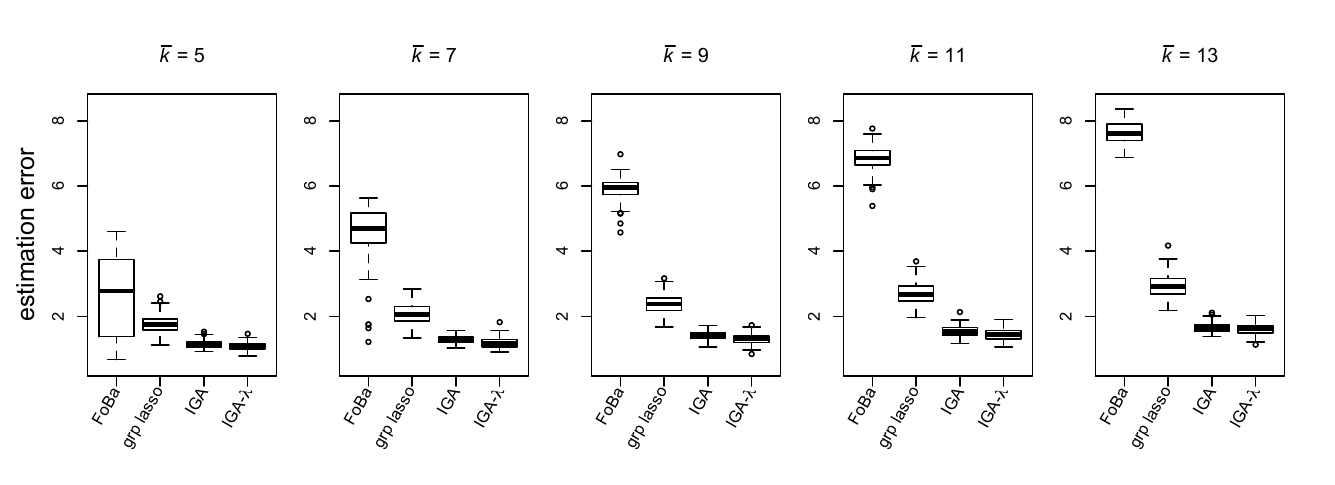}
\caption{Boxplots of estimation errors $\norm{\hat\w-
\w^*}$ for Case 1. Upper panel: $\beta=0.4$; Lower panel: $\beta=1$.} 
\label{fig:least_square}
}
\vspace{-.1in}
\end{figure}

\begin{table}[!ht]
\caption{Averaged simulation results for Case 2 based on 100 runs.}
\begin{center}
\scalebox{0.9}{
\begin{tabular}{lcccccccccc}
\toprule 
& \multicolumn{5}{c}{$\beta=0.4$} & \multicolumn{5}{c}{$\beta=1$} \tabularnewline
\cmidrule(lr){2-6} \cmidrule(lr){7-11}
\multicolumn{1}{c}{$\bar k$} & 5 & 7 & 9 & 11 & 13 & 5 & 7 & 9 & 11 & 13 \tabularnewline
\midrule
 \multicolumn{2}{l}{$\norm{\hat\w-\w^*}$}   &  &  &  &  &  &  &  &  &  \tabularnewline
\cmidrule(lr){1-2}
FoBa & 1.95  & 2.32 & 2.64 & 2.93  & 3.18
  & 4.43  & 5.46  & 6.35  & 7.09 & 7.80\tabularnewline
&  (0.01)  & (0.01) & (0.01) & (0.01)  &  (0.01)
  & (0.02)  & (0.02)  & (0.02)  & (0.02) & (0.02)\tabularnewline
group Lasso & 1.68  & 2.04  & 2.33 & 2.63 
  & 2.90  & 3.90  & 4.93  & 5.83  & 6.60
  & 7.36 \tabularnewline
 & (0.01)  & (0.01)  & (0.01) & (0.01) & (0.01)  & (0.02)  & (0.03)  & (0.02)  & (0.03)
  & (0.02)  \tabularnewline
IGA & 1.28 & 1.66  & 2.05 & 2.42  & 2.86
  & 2.23  & 2.99  & 3.82 & 4.60  & 5.51 \tabularnewline
 & (0.03)& (0.03)  & (0.03)  & (0.03)  & (0.03)
  & (0.02)  & (0.02)  & (0.03)  & (0.03)  & (0.05)  \tabularnewline
IGA-$\lambda$ & 1.16 & 1.50 & 1.79  & 2.12
  & 2.56 & 2.22  & 2.97  & 3.76  & 4.50
  & 5.25  \tabularnewline
 & (0.02)  & (0.03) & (0.03)  & (0.02)
  & (0.02)  & (0.01)  & (0.02)  & (0.02)  &  (0.03)
  & (0.03)  \tabularnewline
GIGA & 1.45 & 1.85 & 2.16  & 2.55
  & 2.87 & 2.28  & 3.32  & 4.35  & 5.42
  & 6.38  \tabularnewline
 & (0.03)  & (0.02) & (0.03)  & (0.02)
  & (0.02)  & (0.03)  & (0.05)  & (0.06)  &  (0.07)
  & (0.07)  \vspace{0.05in}\tabularnewline

\multicolumn{2}{l}{$\abs{\hat G\cap \bar G}$}   &  &  &  &  &  &  &  &  &  \tabularnewline
\cmidrule(lr){1-2}
FoBa & 2.25 & 2.74   & 2.57 & 2.38  & 2.68  & 4.16  & 4.69  & 4.60 &
                                                                     4.88  & 4.58  \tabularnewline
group Lasso & 4.00  & 5.55  & 7.07  & 8.44 & 9.44  & 4.97  & 6.83  &
  8.69 & 10.41  & 11.82 \tabularnewline
IGA & 3.78   & 4.90  & 5.46  & 5.29  & 4.47  & 5.00  & 6.96  & 8.68
  & 10.36  & 11.30 \tabularnewline
IGA-$\lambda$ & 4.35 & 5.82  & 7.03 & 8.07  & 7.81  & 5.00  & 6.99
  & 8.84  & 10.68  & 12.16  \tabularnewline
GIGA & 3.03 & 3.75  & 4.40 & 4.04  & 4.28  & 4.88  & 6.41
  & 7.45  & 7.89  & { 8.20}  \vspace{0.05in}\tabularnewline

 \multicolumn{2}{l}{$\abs{\hat G\backslash \bar G}$}   &  &  &  &  &  &  &  &  &  \tabularnewline
\cmidrule(lr){1-2}
FoBa & 0.50  & 0.52  & 0.49 &0.34  & 0.33 & 1.41 & 1.15  & 1.25  & 0.94  & 0.56 \tabularnewline
group Lasso & 7.52  & 9.88  & 13.07  & 14.25  & 14.67  & 18.68  &
                                                                    22.20
  & 23.94  & 25.37  & 23.59  \tabularnewline
IGA & 0.29  & 0.46  & 0.50  & 0.41  & 0.55  & 1.36  & 1.08 & 1.10  &
                                                                     1.21  & 1.60  \tabularnewline
IGA-$\lambda$ & 0.52 & 0.62  & 0.55  & 0.42  & 0.64  & 1.63  &1.41 &     1.40& 1.41  & 1.35  \tabularnewline
GIGA & 0.35 & 0.32  & 0.33  & 0.24  & 0.32  & 1.50  & 1.41 &     1.43 & 1.42  & 1.35  \tabularnewline
\bottomrule
\end{tabular}
}
\end{center}
\label{tab:sim2}
\end{table}
\begin{figure}[!ht]
\vspace{-.1in}
{
\centering
\includegraphics[scale=.65]{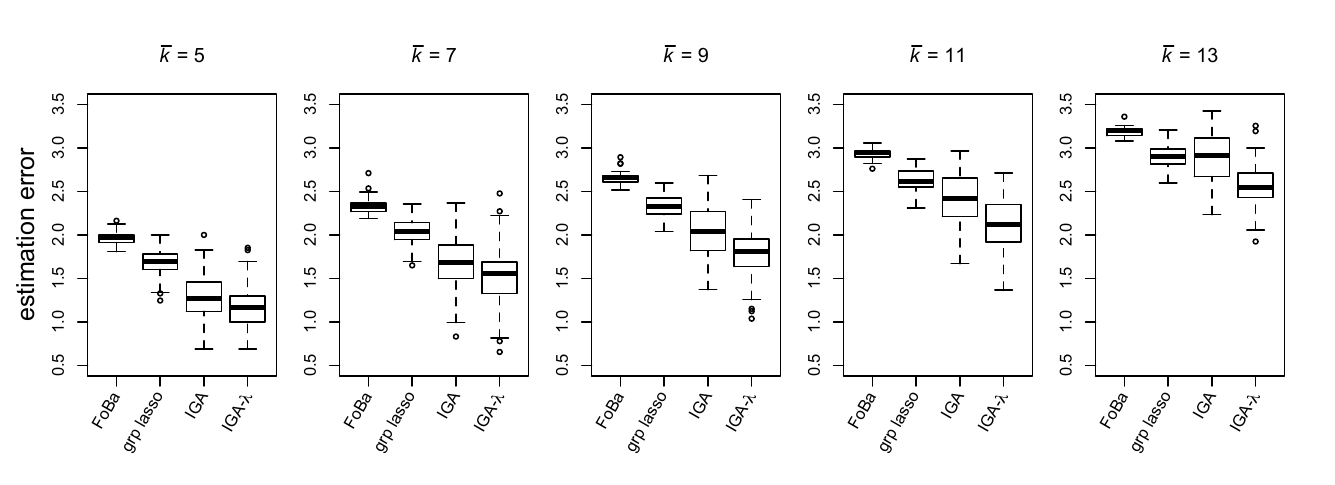}
\vspace{0in}
\includegraphics[scale=.65]{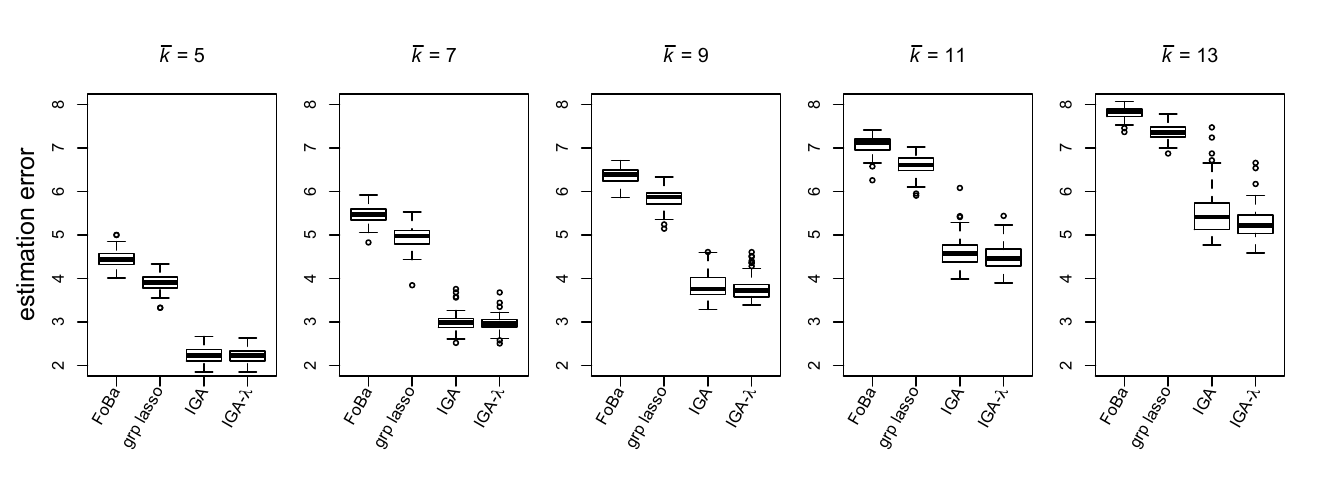}
\caption{Boxplots of estimation errors $\norm{\hat\w-
\w^*}$ for Case 2. Upper panel: $\beta=0.4$; Lower panel: $\beta=1$.} 
\label{fig:logistic}
}
\vspace{-.1in}
\end{figure}

We considered both sparse linear model (Case 1) and sparse logistic regression
(Case 2)
settings. In both cases, assume feature dimension  $p=1000$, which consists of $m=200$
non-overlapping groups with $q=5$ elements in each group. Suppose each
feature vector  $\x \sim N(\mathbf
0_p,\Sigma)$, where elements of $\Sigma$ have the exponential decay
structure $(\Sigma)_{ij}=\rho^{\abs{i-j}}$ with $\rho=0.5$ ($1\leq
i,j\leq p$). Given a true
group sparsity level $\bar k$, we define the set of relevant groups
$\bar G$ to be $\bar G=\{1,3,\cdots, 2\bar k-1\}$.  Then assume linear model
\begin{equation*}
\text{Case 1: }\quad  y=\sum_{g\in \bar G} \w_g^T \x_g+\varepsilon,
\end{equation*}
where each element in coefficient  $\w_g$ 
follows uniform distribution $U(-\beta,\beta)$, and random error  $\varepsilon$
follows $N(0,2)$.  For logistic regression, assume the link
function
\begin{equation*}
\text{Case 2: }\quad  \log\frac{\mu}{1+\mu}=\sum_{g\in \bar G} \w_g^T \x_g,
\end{equation*}
where $\mu=\P(y=1\given\x)=1-\P(y=-1\given\x)$, and the
coefficient  $\w_g$ is generated the same way as Case 1. We
considered different group sparsity levels  $\bar k = 5, 7, 9, 11, 13$
and signal strengths  $\beta=0.4, 1.0$ for both cases.
With sample size $n=300$, we repeated the experiment 100 times and
summarized the averaged results of Case 1 and Case 2 in
Table~\ref{tab:sim1} and Table~\ref{tab:sim2},
respectively (numbers in parenthesis are standard errors). In addition, we created side-by-side
boxplots of the estimation errors in Figure~\ref{fig:least_square} for Case 1
and Figure~\ref{fig:logistic} for Case 2. 

The results of Case 1 in Table~\ref{tab:sim1} and
Figure~\ref{fig:least_square} showed that IGA and IGA-$\lambda$
 performs very competitively compared to FoBa and the group Lasso in
coefficient estimation, which is not surprising given that
FoBa does not take advantage of the benefits of group sparsity, and
the group Lasso has more estimation bias for relatively large
coefficients. With help from the expert priority list,
IGA-$\lambda$ performed the best in this example. In group selection
performance, the group
Lasso showed the tendency to have larger number of incorrectly
selected groups
$\abs{\hat G\backslash \bar G}$ than that of IGA.  

The results of Case 2 in Table~\ref{tab:sim2} and
Figure~\ref{fig:logistic} also showed that IGA and IGA-$\lambda$  give largely satisfactory performance compared to
FoBa and the group Lasso. Interestingly,
by comparing the two different signal strength choices (upper panel
vs. lower panel in Figure~\ref{fig:logistic}), we  observed that
the relative difference between IGA (or IGA-$\lambda$) and the group
Lasso in 
estimation error  appears to widen as we increased the signal strength  $\beta$ from 0.4 to
1. This
observation matched our expectation that 
IGA can become more favorable to the group Lasso when there are more
feature groups with relatively strong signals (that is,  $\norm{\w_g^*}>\Omega{\left(\sqrt{
  {q+\log (m)\over n}}\right)}$).   In
addition, when $\beta=0.4$, besides improving the coefficient estimation, IGA-$\lambda$ selected more relevant groups than IGA in
$\abs{\hat G\cap\bar G}$ at a relatively
small expense of $\abs{\hat G\backslash \bar G}$. Like in Case 1, the group Lasso
selected more irrelevant groups than IGA and IGA-$\lambda$. The
widened estimation error difference between IGA (or IGA-$\lambda$) and
the group Lasso with increased signal strength were similarly observed in Case 1
under a sample size of  $n=200$. Its detailed numerical results are summarized in Table~\ref{tab:sim1_200}.

\begin{table}[!ht]
\caption{Averaged simulation results for Case 1 based on 100 runs ($n=200$).}
\begin{center}
\scalebox{0.9}{
\begin{tabular}{lcccccccccc}
\toprule 
& \multicolumn{5}{c}{$\beta=0.4$} & \multicolumn{5}{c}{$\beta=1$} \tabularnewline
\cmidrule(lr){2-6} \cmidrule(lr){7-11}
\multicolumn{1}{c}{$\bar k$} & 5 & 7 & 9 & 11 & 13 & 5 & 7 & 9 & 11 & 13 \tabularnewline
\midrule
 \multicolumn{2}{l}{$\norm{\hat\w-\w^*}$}   &  &  &  &  &  &  &  &  &  \tabularnewline
\cmidrule(lr){1-2}
FoBa & 2.06  & 2.50 & 2.80 & 3.13  & 3.41
  & 4.47  & 5.62  & 6.62  & 7.53 & 8.36\tabularnewline
&  (0.01)  & 
 (0.01) & (0.01) & (0.01)  &  (0.02)
  & (0.03)  & (0.03)  & (0.03)  & (0.02) & (0.03)\tabularnewline
group Lasso & 1.63  & 1.87  & 2.17 & 2.42 
  & 2.58  & 2.26  & 2.65  & 3.20  & 3.83
  & 4.30 \tabularnewline
 & (0.02)  & (0.02)  & (0.02) & (0.02)
  & (0.02)  & (0.04)  & (0.04)  & (0.04)  & (0.05)
  & (0.06)  \tabularnewline
IGA & 1.40 & 1.65  & 2.06 & 2.53  & 2.88
  & 1.46  & 1.65  & 1.81 & 1.95  & 2.19 \tabularnewline
 & (0.03)& (0.03)  & (0.03)  & (0.04)  & (0.04)
  & (0.02)  & (0.02)  & (0.02)  & (0.02)  & (0.03)  \tabularnewline
IGA-$\lambda$ & 1.29 & 1.51 & 1.71  & 2.20
  & 2.57 & 1.37  & 1.52  & 1.71 & 1.92  & 2.14  \tabularnewline
 & (0.03)  & (0.03) & (0.02)  & (0.04)
  & (0.03)  & (0.02)  & (0.02)  & (0.02)  &  (0.02)
  & (0.03)  \tabularnewline
GIGA & 1.52 & 1.81 & 2.22  & 2.55
  & 2.93 & 1.29  & 1.46  & 1.75  & 1.98
  & 2.76  \tabularnewline
 & (0.03)  & (0.03) & (0.03)  & (0.03)
  & (0.03)  & (0.02)  & (0.03)  & (0.05)  &  (0.05)
  & { (}0.10)  \vspace{0.05in}\tabularnewline

\multicolumn{2}{l}{$\abs{\hat G\cap \bar G}$}   &  &  &  &  &  &  &  &  &  \tabularnewline
\cmidrule(lr){1-2}
FoBa & 2.17 & 2.99   & 2.90 & 3.20  & 3.39  & 4.58  & 6.01  & 6.36  &
                                                                      6.80  & 6.99  \tabularnewline
group Lasso & 3.62  & 5.60  & 6.83  & 8.37 & 10.25  & 5.00  & 6.98  &
  9.00 & 10.93  & 12.89 \tabularnewline
IGA & 3.75   & 5.58  & 6.15  & 6.00  & 6.76  & 5.00  & 7.00  & 9.00
  & 11.00  & 12.99 \tabularnewline
IGA-$\lambda$ & 4.34 & 6.30  & 7.67 & 8.62  & 9.33  & 5.00  & 7.00
  &9.00  & 11.00  & 13.00  \tabularnewline
GIGA & 3.12 & 4.21  & 4.85 & 5.30  & 5.91  & 4.99  & 6.98
  &8.90  & 10.84  & { 1}2.19 \vspace{0.05in}\tabularnewline

 \multicolumn{2}{l}{$\abs{\hat G\backslash \bar G}$}   &  &  &  &  &  &  &  &  &  \tabularnewline
\cmidrule(lr){1-2}
FoBa & 0.61  & 0.96  & 0.63 & 0.81  & 1.93 & 1.40 & 2.28  & 1.86  & 1.73  & 1.60 \tabularnewline
group Lasso & 4.79  & 7.39  & 9.28  & 12.01  & 21.24  & 14.46  &
                                                                   20.75
  & 26.59  & 30.63  & 34.25  \tabularnewline
IGA & 0.44  & 0.52  & 0.52  & 0.60  & 0.83  & 1.84  & 1.87 & 1.75  &
                                                                     1.57  & 1.48  \tabularnewline
IGA-$\lambda$ & 0.71 & 0.77  & 0.80  & 0.85  & 0.82  & 1.93  & 1.95  &
                                                                       1.91
  & 1.81  & 1.63  \tabularnewline
GIGA & 0.74 & 0.57  & 0.80  & 0.61  & 0.83  & 1.80  & 2.00  &
                                                                    1.85
  & 1.51  & 1.06  \tabularnewline
\bottomrule
\end{tabular}
}
\end{center}
\label{tab:sim1_200}
\end{table}

\begin{table}[!ht]
\caption{Time (in seconds) to fit a simulated Case 1 data set.}
\begin{center}
\scalebox{0.9}{
\begin{tabular}{lcccc}
\toprule
  & $n=2000$ & $n=5000$ & $n=10000$ & $n=20000$\\
\midrule
group Lasso & 12.3 & 25.3 & 49.9 &  98.0\\
IGA & 6.1 & 18.5 & 40.9 &  91.4 \\
IGA-$\lambda$ & 17.6 & 62.2 & 142.2 &  321.3 \\
GIGA & 0.8 & 2.1 & 4.0 &  9.1 \\
\bottomrule
\end{tabular}
}
\end{center}
\label{table:timing}
\end{table}

For scalability, we also enlarged the sample size to $n=2000, 5000, 10000,
20000$ for
Case 1 while holding everything else constant with  $p=1000$ and  $\bar k=5$. We listed the times of one-run experiment in
Table~\ref{table:timing}, where each fold of CV was computed in a
{\it parallel} fashion (Intel Xeon W 3.0GHz CPU). The computational
times for IGA were comparable to those of the group
Lasso when both were implemented in MATLAB on the same machine (we
omit FoBa due to the known grouping structure). As
expected, the computational burden was higher for  IGA-$\lambda$ than
IGA because the CV stage of
IGA-$\lambda$ involves five candidate values of
$\lambda$ (as opposed to $\lambda=1$ in IGA) and
one more group selection path with full data. 

In addition, as described in Section~\ref{subsec:giga}, we proposed the gradient-based
IGA algorithm (or GIGA for brevity) to facilitate the computational efficiency for the
forward steps. To evaluate the GIGA algorithm, we repeated the
simulation experiment before for Case 1 and Case 2. The estimation
errors summarized in Tables~\ref{tab:sim1}--\ref{tab:sim1_200} showed
that GIGA still performed competitively compared to the benchmark
methods like FoBa and the group Lasso. More importantly, since GIGA
avoids repeatedly performing optimization of the criterion functions
in the forward step and only requires computation of the gradients
instead, the averaged times given in Table~\ref{tab:giga_speedup}
indeed showed that GIGA significantly reduced the computation
time compared to IGA. We also applied GIGA in the
scalability study and observed from Table~\ref{table:timing} that times of GIGA were only
about 1/10 of IGA. The numerical experience above confirms that the proposed
GIGA algorithm can be a promising variant of the plain vanilla IGA
when computational efficiency is a practical concern.

\begin{table}[!ht]
\caption{Averaged time (in seconds) of GIGA vs. IGA to fit a simulated
  data set with $n=300$.}
\begin{center}
\scalebox{0.9}{
\begin{tabular}{clcccccccccc}
\toprule 
& & \multicolumn{5}{c}{$\beta=0.4$} & \multicolumn{5}{c}{$\beta=1$} \tabularnewline
\cmidrule(lr){3-7} \cmidrule(lr){8-12}
Case & \multicolumn{1}{c}{$\bar k$} & 5 & 7 & 9 & 11 & 13 & 5 & 7 & 9 & 11 & 13 \tabularnewline
\midrule
 1 & IGA & 3.5   & 3.4  & 5.5  & 3.7  & 2.5  & 2.3  & 2.2  & 2.4
  & 2.2  & 1.8 \tabularnewline
 & GIGA & 0.4 & 0.3  & 0.3 & 0.3  & 0.3  & 0.3  & 0.2
  & 0.2  & 0.2  & { 0.2}  \vspace{0.05in}\tabularnewline

2 & IGA & 34.4  & 32.8  & 33.2  & 33.0  & 33.4  & 26.9  & 26.8 & 25.4  &
                                                                     25.3  & 26.7  \tabularnewline
 & GIGA & 1.9 & 1.8  & 1.8  & 1.8  & 1.8  & 3.5  & 3.3 &     3.4 & 3.3  & 3.3  \tabularnewline
\bottomrule
\end{tabular}
}
\end{center}
\label{tab:giga_speedup}
\end{table}

\section{Sensor Selection for Human Activity Recognition}\label{sec:app}
One important industrial application of group sparsity learning 
is in sensor selection problems. We are particularly interested in
the application for recognizing human  activity at home,
and the aim is to reduce the number of deployed sensors without 
significant accuracy reduction in activity recognition. 
In the real experiment, we deployed 40 sensors with 14 considered
activity categories. We used pyroelectric infrared sensors, which returned binary signals in reaction to human motion.
Figure~\ref{fig:sensor_location} shows our experimental room layout and
sensor locations \citep{liu2013forward}. The number stands for sensor ID and the circle approximately represents 
the area covered by the sensor.
As can be seen, 40 sensors have been deployed, which returned 40-dimensional 
binary time series data. There were 14 pre-determined human activity categories. The numbers of training samples
and testing samples were roughly the same with the approximate sizes
of 270K each. The labels of testing data
were blind to the data analysts, and we were only allowed to
submit the prediction results to the internal server owned by NEC Corporation to query the
prediction accuracy, without direct access to testing samples. Detailed
information on activity categories and sample size is summarized in
Table~\ref{tab:activities}.

\begin{figure}[htp!]
\centering
\includegraphics[width=0.5\linewidth]{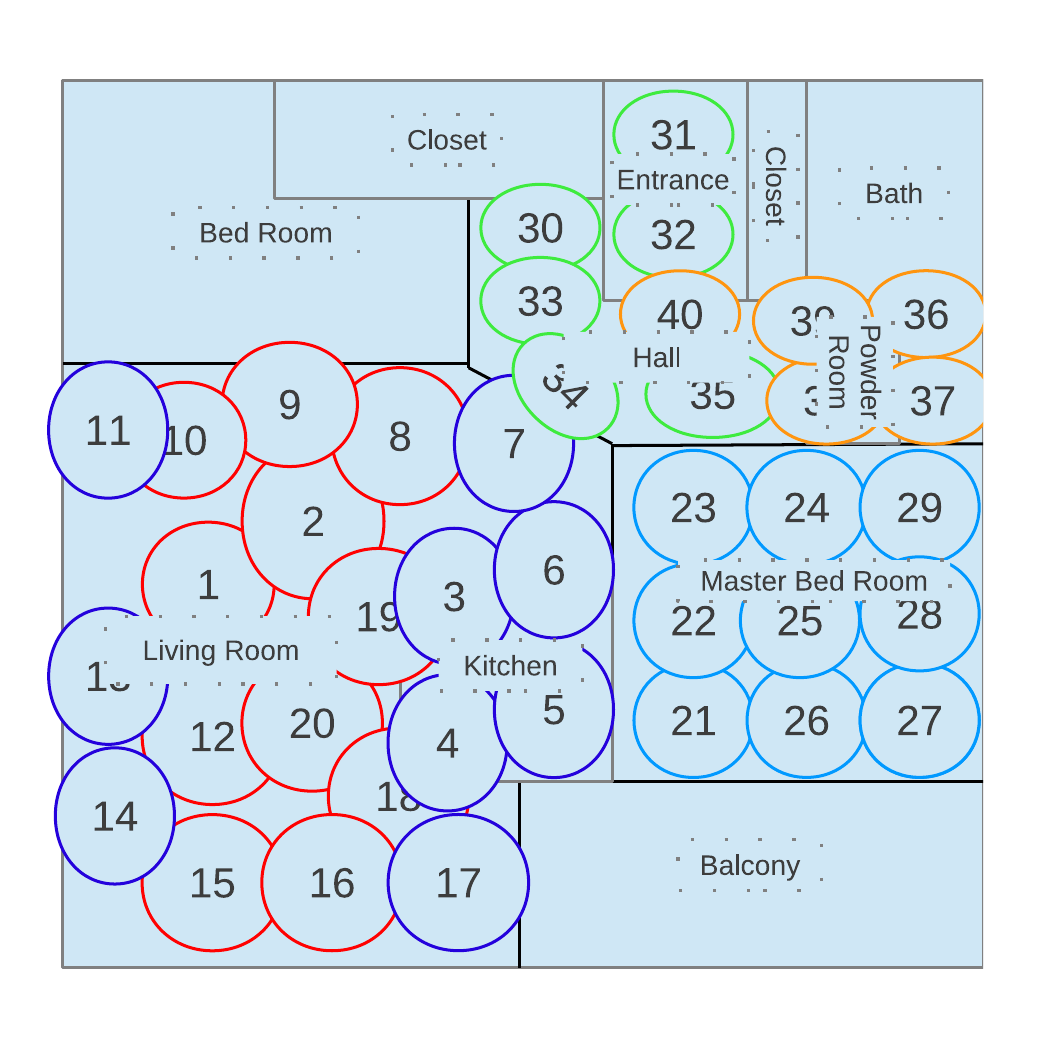}
\caption{Room layout and sensor locations.}
\label{fig:sensor_location}
\end{figure}

\begin{table}[htp!]
\caption{Activities in the sensor data set.}
\label{tab:activities}
\begin{center}
\scalebox{0.9}{
\begin{tabular}{c c c}
\toprule
 ID& Activity & train / test samples\\
\midrule
1& Sleeping & 81K / 87K \\
2& Out of Home & 66K / 42K \\
3& Using Computer& 64K 46K \\
4& Relaxing & 25K / 65K \\
5& Eating & 6.4K / 6.0K \\
6& Cooking & 5.2K / 4.6K \\
7& Showering (Bathing) & 3.9K / 45.0K \\
8& No Event & 3.4K / 3.5K \\
9& Using Toilet & 2.5K / 2.6K \\
10& Hygiene (brushing teeth, etc)& 1.6K / 1.6K \\
11& Dishwashing & 1.5K / 1.8K \\
12& Beverage Preparation & 1.4K / 1.4K \\
13& Bath Cleaning / Preparation & 0.5K / 0.3K \\
14& Others & 6.5K / 2.1K \\
\midrule
Total & - & 270K /270K \\
\bottomrule
\end{tabular}
}
\end{center}
\end{table}

Two types of features were created: activity-signal features~($40\times 14 = 560$) and activity-activity features ($14\times 14 = 196$). Given
the aim of sensor number reduction, enforcement of sparseness to
individual features can be rather inefficient. Accordingly, we created
40 groups on activity-signal features and each group contained features
related to one sensor. Different from the standard independent settings
of \eqref{eq:groupl0}, we considered here a more general form of the criterion
function $Q(\w)$ by using the negative log-likelihood of a linear-chain conditional random
 fields model (CRFs, \citealp{lafferty2001conditional}).

Specifically, given sample size $N$, let
$\y=(y_1,\cdots, y_N)^T$ be the sequence of action labels with
$y_t\in\{1,\cdots,14\}$, and let  $Z=(\z_1,\cdots,\z_N)^T$ be the
sequence of sensor binary signals with
$\z_t=(z_{t,1},\cdots,z_{t,40})^T\in\{0,1\}^{40}$. Then define
the activity-signal features $f_{i,j}(y_t,\z_t) = I(y_t=j,
z_{t,i}=1)$ and the associated sensor coefficients  $v_{i,j}$ for
$1\leq i \leq 40$, $1\leq j\leq 14$ and  $1\leq t\leq N$;
define the activity-activity features $g_{i,j}(y_t,y_{t+1})=I(y_t=i,
y_{t+1}=j)$ and the associated activity transition coefficients $u_{i,j}$ for $1\leq
i,j\leq 14$ and $1\leq t\leq N$. For notation
convenience, set $y_0=0$. The coefficient vector
$\mathbf v_i=(v_{i,1},\cdots, v_{i,14})^T$ then corresponds to the group of sensor
$i$ $(1\leq i\leq 40)$. With the targeted goal of sensor selection, we
can impose group sparsity on the coefficient
$\mathbf v=(\mathbf v_1^T,\cdots,\mathbf v_{40}^T)^T$ since it has a clear grouping structure
by construction with 40 groups, while the transition coefficient $\mathbf
u=(u_{1,1},u_{1,2},\cdots, u_{14,13}, u_{14,14})$ is assumed to be
dense. Accordingly, the complete coefficient vector is  $\w=(\mathbf v^T,\u^T)^T$,
and the group selection should focus on the components in $\mathbf v$
while allowing elements in $\u$ to be nonezero. Then motivated by
Hidden Markov models, the linear-chain CRF objective considers the conditional probability 
\begin{align}
p_{\w}(\y \given Z) &= {\frac{1}{q_\w(Z)}} \exp\Bigl\{\sum_{t=1}^{N}
  \sum_{i=1}^{40} \sum_{j=1}^{14} v_{i,j} f_{i,j}(y_t, \z_t)
  +\sum_{t=1}^{N} \sum_{i=1}^{14}\sum_{j=1}^{14} u_{i,j} g_{i,j}(y_{t-1},
  y_{t}) \Bigr\}\notag\\
& =: \frac{1}{q_\w(Z)} \exp\Bigl\{\sum_{t=1}^N \bigl( \mathbf
  v^T\mathbf f (y_t,\z_t) + \u^T\mathbf g(y_{t-1}, y_t) \bigr) \Bigr\}\notag\\
& =: \frac{1}{q_\w(Z)}\tilde p_{\w}(\y \given Z), \label{eq:condprob}
\end{align}
where  $q_\w(Z)=\sum_{\tilde\y}\tilde p_{\w}(\tilde\y \given Z)$ is the
normalization factor, and  $\mathbf f$ and $\mathbf g$ are vector-valued functions
consisting of $(f_{1,1},f_{1,2},\cdots, f_{40,13}, f_{40,14})^T$ and
$(g_{1,1},g_{1,2},\cdots, g_{14,13}, g_{14,14})^T$, respectively. The resulting negative log-likelihood
criterion function is  $Q(\w) = -\log p_{\w}(\y\given
Z)$, which can be shown to be both smooth and convex. Given parameter
$\w$, the criterion function's gradient and
the normalization factor (and thus $Q(\w)$ itself) can be computed
efficiently with a polynomial time of $N$ due to the chain-structured
graph and the associated recursive computing scheme
\citep[Chapter 4.1]{sutton2012introduction}; the
inference for testing data $\tilde Z$ can be similarly obtained by maximizing
$p_{\w}(\tilde\y | \tilde Z)$ with respective to  $\tilde\y$. In
addition, optimization of the criterion function with respect to  $\w$
can be achieved by the generic maximum likelihood estimation
\citep[Chapter 5.1]{sutton2012introduction}. With the techniques above, for group/sensor
selection purposes, the IGA algorithm can be naturally
applied to the criterion function from the negative
log-likelihood of \eqref{eq:condprob}. In the following experiment,
our IGA approach for sensor selection adopted the
gradient-based variant as discussed in Section~\ref{subsec:giga} and 
performed the forward-backward selection for groups in $\mathbf v$.

We also compared IGA with two other benchmark sensor selection methods: gradient
FoBa~(FoBa) and  Group L1 CRF (GroupLasso). For all three methods, we directly set the
numbers for selected sensors (each sensor corresponds to one group of
features), as suggested in \citet{zhang2011adaptive} and
\citet{liu2013forward}. In our specific problem, it is considered
informative and convenient to know the classification performance
given a specified number of sensors with the practical goal of reducing the number of required sensors (even at a small expense of accuracy). The overall classification error rates on testing samples with varying
number of sensors are summarized in
Figure~\ref{fig:result_activities1}(a). We discovered the following. 
\begin{itemize}
\item
When the number of sensors was relatively small (5-9), IGA
outperformed the other methods. With the sufficient
number of sensors, FoBa also performed competitively. We observed big
accuracy improvement for FoBa around 10-11 sensors. The lists of selected
sensors in Table~\ref{tab:selected_sensors}, in together with the corresponding
classification results in Figures~\ref{fig:result_activities1}(a),
seemed to suggest that   
the features related to Sensor 28 were important (for a justification,
see also discussion below on individual activities). The IGA
successfully chose this sensor in  early iterations while FoBa
used longer iterations.
\item
IGA required only 7-8 sensors to achieve nearly best performance
though FoBa required 11-12 sensors. Therefore, we may reduce 4-5 sensors by using IGA.
GroupLasso gradually reduced the classification error with increasing
number of sensors, but its classification performance was not ideal  in this study when compared to
the considered greedy methods.
\end{itemize}

\begin{figure}[ht!]
\centering
\subfigure[]{
\includegraphics[width=0.6\linewidth]{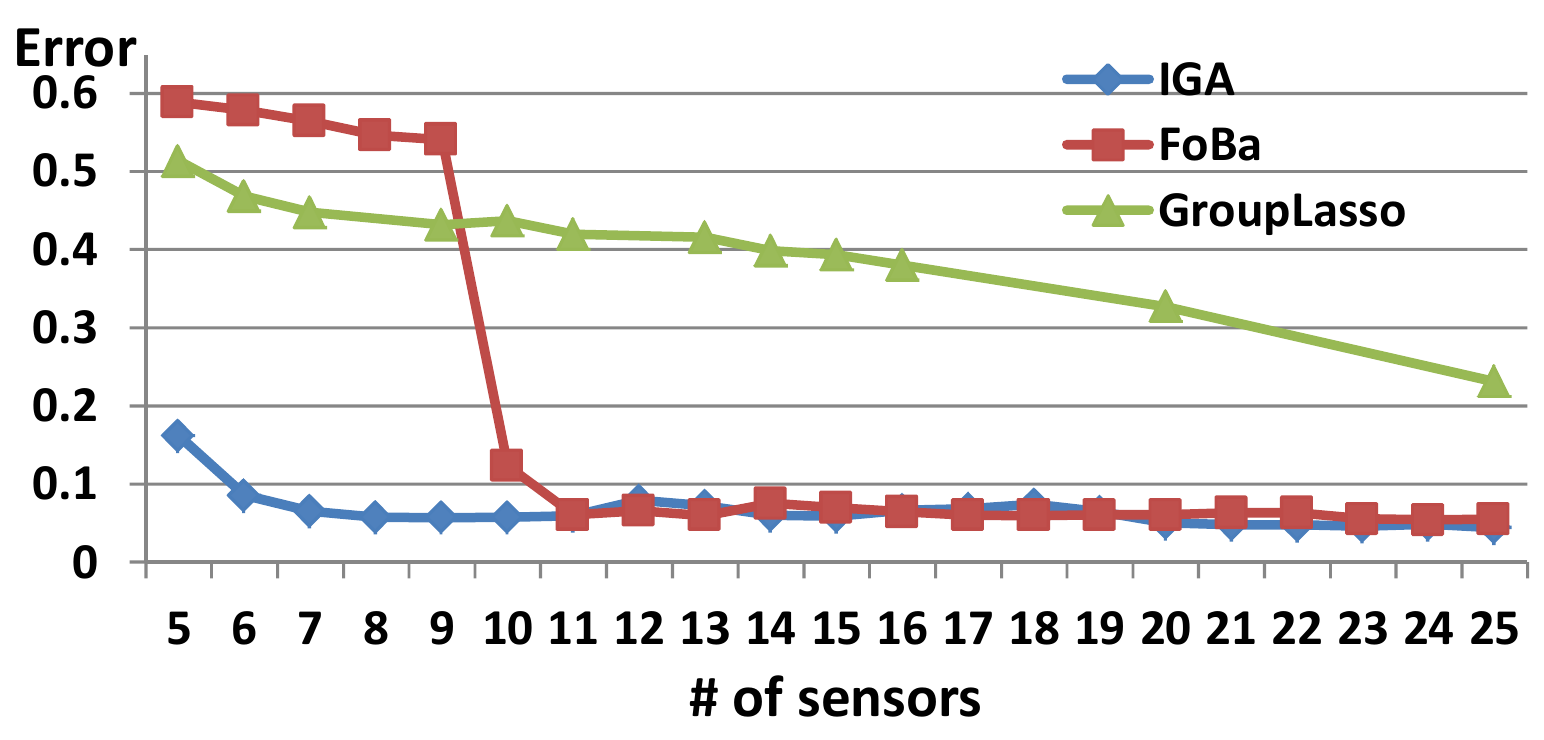}
}
\subfigure[]{
\includegraphics[width=0.45\linewidth]{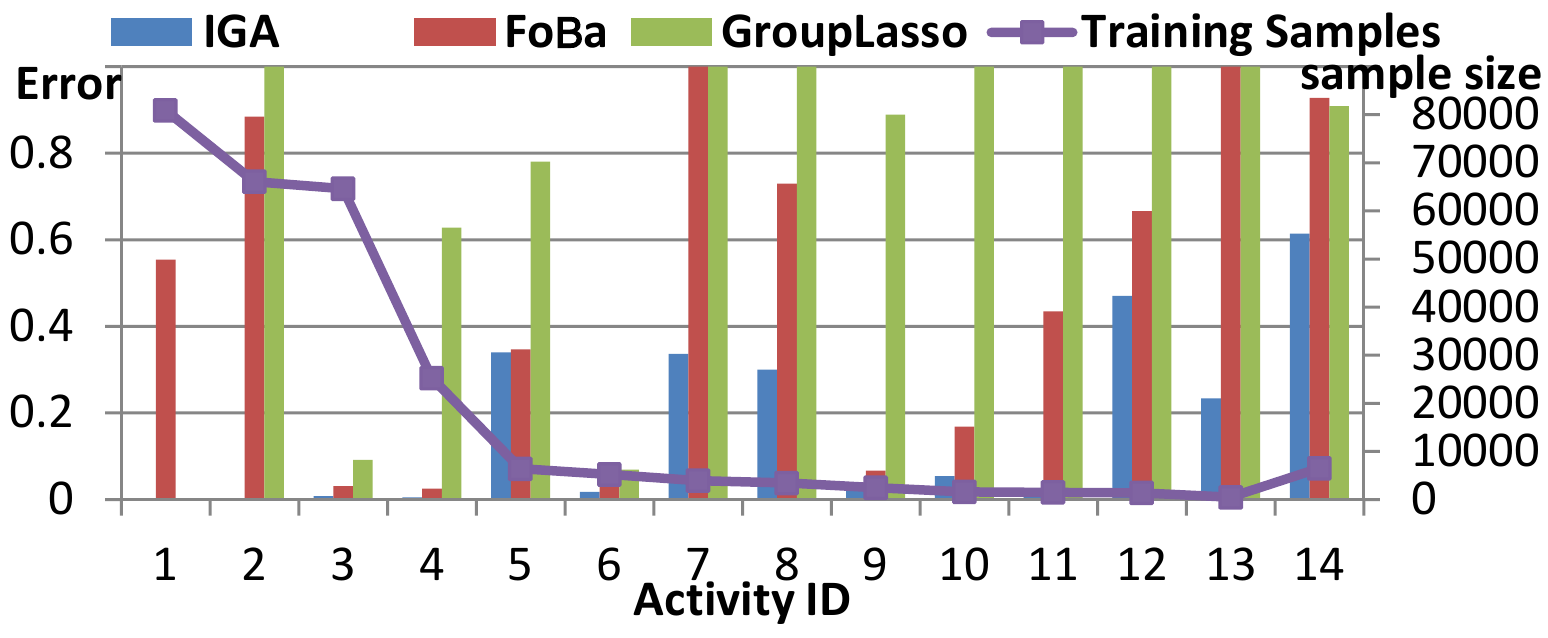}
}
\subfigure[]{
\includegraphics[width=0.45\linewidth]{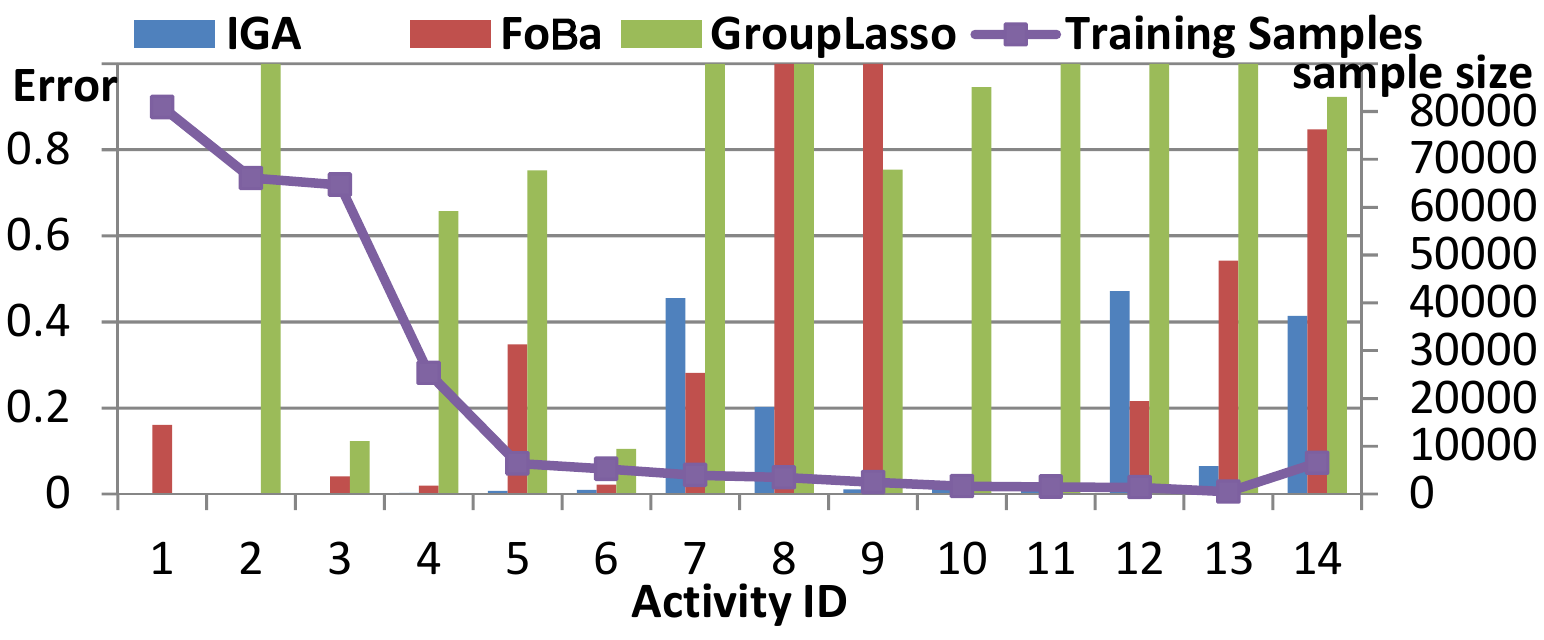}
}
\caption{Comparison among FoBa, GroupLasso, and IGA. (a) 
  Classification errors given the number of selected sensors for
  testing samples (horizontal axis represents number of selected
  sensors); (b) Activity-wise testing classification errors
  with 7 sensors (horizontal axis represents activity category ID; solid line is training sample size); (c) Activity-wise testing  classification errors with 10
  sensors (horizontal axis represents activity category ID; solid line is training sample size).} 
\label{fig:result_activities1}
\end{figure}

\begin{table}[htp!]
\caption{Selected sensors.}
\label{tab:selected_sensors}
\begin{center}
\begin{tabular}{c c c}
\toprule
Method & 7 sensors & 10 sensors \\
\midrule
IGA & $\{5,8,10,19,28,39,40\}$ & $\{7 \,\text{sensors}\} \cup \{1,31,35\}$ \\
FoBa &  $\{1,5,9,10,19,39,40\}$ & $\{7 \,\text{sensors}\} \cup \{8,28,35\}$   \\
GroupLasso & $\{1,5,9,10,19,39,40\}$ & $\{7 \,\text{sensors}\} \cup \{2,36,37\}$ \\
\bottomrule
\end{tabular}
\end{center}
\end{table}

The classification errors for individual activities with 7
sensors and 10 sensors are shown in Figure~\ref{fig:result_activities1}(b) and Figure~\ref{fig:result_activities1}(c), respectively.
We can see that for most of the activity categories, with either 7 or 10 sensors,
IGA performed generally better than the other two alternatives.
Interestingly, recognition errors of FoBa for the activities, 
{\it Sleeping}~(Activity ID=1) and {\it Out of Home}~(Activity ID=2), were 
quite poor with 7 sensors, but the errors were drastically reduced with 10 sensors. 
Since both the {\it Sleeping} and  {\it Out of Home}  activities have ``quiet'' movements and generate little sensor signals, 
they are expected to be difficult to distinguish from each other. Our
experiment results in Table~\ref{tab:selected_sensors} showed that FoBa added Sensor~28 at ``the number of sensors = 10'' and 
its recognition error dramatically improved.
This observation made practical sense as Sensor~28 was located near
the bed (shown in Figure~\ref{fig:sensor_location}) and therefore played a key role to distinguish these two 
activities. This reaffirms our discussion above that Sensor~28 could be an
important sensor in practice; with noisy signals from the ``quiet'' movements,  IGA 
discovered this sensor at earlier steps than FoBa. This experiment
showed the promising use of IGA in selecting a small number of sensors
while maintaining competitive activity recognition accuracy.

\section{Conclusion}

In this paper, we propose a new interactive greedy algorithm designed
to exploit known grouping structure for improved estimation and
feature group recovery performance. Motivated under a general setting
applicable to a broad range of applications, we provide the
theoretical and empirical investigations on the proposed algorithm. This study establishes explicit estimation error bounds
and group selection consistency results in high-dimensional linear
model and logistic regression, and supports that this new algorithm
can be a promising and flexible group sparsity learning tool in
practice. In the future, it is of interest to conduct further study
under the general setting and give explicit performance guarantee for
a broader class of GLM models beyond linear and logistic
regression. It is also promising to study the performance of our
approach under challenging settings including sequential decision making
(e.g., \citealp{qian2016kernel}) and non-smooth objective function
(e.g., \citealp{gu2018admm}) problems.


\section*{Supplementary Materials}

{
\begin{description}
\item[Supplement to ``An Interactive Greedy Approach to Group Sparsity
  in High Dimension''] \mbox{}\\
  We provide the proofs of Theorems~\ref{thm:main}
and \ref{thm:main_giga} in Supplement~\ref{subsec:proof1}. The proofs for the sparse
linear model and the sparse logistic regression are given in
Supplement~\ref{subsec:proof2} and Supplement~\ref{subsec:proof3},
respectively. The useful lemmas for intermediate steps of these proofs are relegated to Supplement~\ref{sec:suppthm}. (supplement.pdf)
\item[MATLAB package] The MATLAB package demonstrates the
  implementation and use of the IGA and GIGA algorithms for both sparse linear
  model and sparse logistic regression. It is available at \url{https://github.com/weiqian1/IGA}. 
\end{description}
}

 \section*{Acknowledgement}

{
 We sincerely thank the Editor, the Associate Editor and two anonymous
 reviewers for their valuable and insightful comments that helped to
 improve this manuscript significantly. 
 Ji Liu's research is partially supported by NSF CCF-1718513, IBM faculty award, and NEC fellowship.
}

\bibliographystyle{agsm}
{\footnotesize
\bibliography{fobacand}
}

\newpage
\renewcommand{\theequation}{A.\arabic{equation}}
\setcounter{equation}{0}  
  

\appendix
\begin{center}
{\Large Supplement to ``An Interactive Greedy Approach to Group Sparsity in High Dimension''}
\end{center}

In this Supplement, we provide the proofs of Theorems~\ref{thm:main}
and \ref{thm:main_giga} in Supplement~\ref{subsec:proof1}, and give the proofs for the sparse
linear model and the sparse logistic regressions in
Supplement~\ref{subsec:proof2} and Supplement~\ref{subsec:proof3},
respectively. We leave useful lemmas to Supplement~\ref{sec:suppthm}.


\section{Proofs of Theorem~\ref{thm:main} and Theorem~\ref{thm:main_giga}}
\label{subsec:proof1}

Theorem~\ref{thm:main} follows directly from statements of
Theorem~\ref{estigood} and Theorem~\ref{thm:main2} in the following.

\begin{theorem}\label{estigood}
Suppose Assumption~\ref{ass:s} holds and let $\delta$  satisfy
$\delta>\frac{4\varphi_{+}(1)}{\lambda\varphi_{-}^{2}(s)}\|\nabla Q(\bar{\w})\|_{G,\infty}^{2}$.
Suppose IGA stops at $k$, and the supporting group is $G^{(k)}$ with
supporting features as $F^{(k)}=F_{G^{(k)}}$.  Then the estimator has the following properties:
\begin{align}
\|\bar{\w}-\w^{(k)}\|  &\leqslant \frac{4\sqrt{ \varphi_{+}(1)\delta}}{\varphi_{-}(s)}\sqrt{|\bar{\Delta}|}
\\
Q(\w^{(k)})-Q(\bar{\w})&\leqslant \frac{2 \varphi_{+}(1)\delta}{\varphi_{-}(s)}|\bar{\Delta}| 
\\
|\bar{G} \setminus G^{(k)} | &\leqslant    2|\bar{\Delta}| 
\\
| G^{(k)} \setminus \bar{G} | &\leqslant    \frac{16{\varphi}^{2}_{+}(1)|\bar{\Delta}|}{\lambda {\varphi}^{2}_{-}(s)}
\end{align}
where $\bar{\Delta}=\{g\in \bar{G} :\| \bar{\w}_{g}\|<\gamma\}$ and 
$\gamma=\sqrt{\frac{16\varphi_{+}(1)\delta}{\varphi_{-}^{2}(s)}}$.
\end{theorem}

\begin{proof}[Proof of Theorem~\ref{estigood}]
After our algorithm terminates, suppose that  we have  $k$  group  selected. From Lemma \ref{decrease}, we can just assume $Q(\w^{(k)})\geqslant Q(\bar{\w})$ and hence we have:
\begin{equation*}
  0\geqslant Q(\bar{\w})-Q(\w^{(k)})
   \geqslant \langle\nabla Q(\w^{(k)}),\bar{\w}-\w^{(k)}\rangle+\frac{\varphi_{-}(s)}{2}\|\bar{\w}-\w^{(k)}\|^{2}
\end{equation*}
After rearranging the above inequality and taking advantage of regrouping, we have
\begin{align*}
  & \frac{\varphi_{-}(s)}{2}\|\bar{\w}-\w^{(k)}\|^{2} \leqslant-\langle\nabla 
	  Q(\w^{(k)}),\bar{\w}-\w^{(k)}\rangle  \\
					      &=-{\langle\nabla
                                                Q(\w^{(k)}),(\bar{\w}-\w^{(k)})_{{\bar{G}{\cup}
                                                G^{(k)}}}\rangle} =-{\langle(\nabla Q(\w^{(k)}))_{{\bar{G}{\cup} G^{(k)}}},(\bar{\w}-\w^{(k)})\rangle}\\
					            &=-{\langle(\nabla
                                                      Q(\w^{(k)}))_{{\bar{G}{\setminus}
                                                      G^{(k)}}},(\bar{\w}-\w^{(k)})\rangle}
                                                      \qquad (\bar{\w}-\w^{(k)} \textrm{ is supported on } \bar{G}{\setminus} G^{(k)} )\\
					      &=-{\langle\nabla
                                                Q(\w^{(k)}),(\bar{\w}-\w^{(k)})_{{\bar{G}{\setminus}
                                                G^{(k)}}}\rangle} \leqslant \|\nabla Q(\w^{(k)})\|_{G,\infty}\|(\bar{\w}-\w^{(k)})_{\bar{G}{\setminus}G^{(k)}}\|_{G,1}\\
    &\leqslant \sqrt{2\varphi_{+}(1)\delta}\sqrt{|\bar{G} {\setminus} G^{(k)}|}\|\bar{\w}-\w^{(k)}\|.  \qquad (\textrm{from Lemma \ref{lemma2}})
\end{align*}

Simplifying the above inequality,  we have the following:
\begin{align}\label{coeferror}
  \|\bar{\w}-\w^{(k)}\| 
   & \leqslant\frac{2\sqrt{2\varphi_{+}(1)\delta}}{\varphi_{-}(s)}\sqrt{|\bar{G} {\setminus} G^{(k)}|}.
\end{align}
We are interested in estimating errors from certain weak channels of our signals, and mathematically we consider a threshold of coefficients so that we can use the Chebyshev 
inequality to get new bounds below:
\begin{align*}
  &\frac{8\varphi_{+}(1)\delta}{\varphi_{-}^{2}(s)}|\bar{G}{{\setminus}}G^{(k)}|
  \geqslant\|\bar{\w}-\w^{(k)}\|^{2}\geqslant \| ( \bar{\w}-\w^{(k)} )_{\bar{G}{{\setminus}}G^{(k)}}\|^{2}\\ 
=&\,\|\bar{\w}_{\bar{G}{\setminus} G^{(k)}}\|^{2} \qquad (\w^{(k)} \textrm{ is supported outside of } \bar{G}{\setminus} G^{(k)})\\
\geqslant&\, \gamma^2 |\{g\in \bar{G} {\setminus} G^{(k)}\ \ | \ \ \|
  \bar{\w}_{g}\|^2\geqslant \gamma^2   \}| \geqslant\frac{16\varphi_{+}(1)\delta}{\varphi_{-}^{2}(s)}|\{g\in \bar{G} {\setminus} G^{(k)}\ \ | \ \ \| \bar{\w}_{g}\|\geqslant\gamma\}|.
\end{align*}
We take $\gamma^2=\frac{16\varphi_{+}(1)\delta}{\varphi_{-}^{2}(s)}$ in the last inequality and hence we can almost get rid of the coefficients: 
\begin{equation*}
|\bar{G}{{\setminus}}G^{(k)}|
\geqslant 2|\{g\in \bar{G} {\setminus} G^{(k)}\ \ | \ \ \|
\bar{\w}_{g}\|\geqslant\gamma\}| =2(|\bar{G}{{\setminus}}G^{(k)}|-|\{g\in \bar{G} {\setminus} G^{(k)}\ \ | \ \ \| \bar{\w}_{g}\|<\gamma\}|).
\end{equation*}

Obviously, this tells a better estimation that we will use often:
\begin{equation*}
|\bar{G}{{\setminus}}G^{(k)}|\leqslant 2|\{g\in \bar{G} {\setminus} G^{(k)}\ \ | \ \ \| \bar{\w}_{g}\| < \gamma\}|
\leqslant 2|\{g\in \bar{G}\ \ | \ \ \| \bar{\w}_{g}\| < \gamma\}|, 
\end{equation*}
which proves the third inequality. Then we can continue our estimation of errors using these new notations. 
First, we put it in (\ref{coeferror}) we have coefficient error bound
\begin{align*}
\|\bar{\w}-\w^{(k)}\| \leqslant\frac{2\sqrt{2\varphi_{+}(1)\delta}}{\varphi_{-}(s)}\sqrt{|\bar{G}{{\setminus}}G^{(k)}|}    \leqslant\frac{4\sqrt{ \varphi_{+}(1)\delta}}{\varphi_{-}(s)}\sqrt{|\{g\in \bar{G} {\setminus} G^{(k)}\ \ | \ \ \|  \bar{\w}_{g}\|<\gamma\}|}.
\end{align*}
Second, we bound the loss function error:
\begin{align*}
  &Q(\bar{\w})-Q(\w^{(k)})
  \geqslant \langle\nabla Q(\w^{(k)}),\bar{\w}-\w^{(k)}\rangle+\frac{\varphi_{-}(s)}{2}{\|\bar{\w}-\w^{(k)}\|^{2}}\\
  	=&\langle\nabla Q(\w^{(k)})_{\bar{G}\cup
           G^{(k)}},({\bar{\w}-\w^{(k)}})_{\bar{G}\cup
           G^{(k)}}\rangle+\frac{\varphi_{-}(s)}{2}{\|\bar{\w}-\w^{(k)}\|^{2}}
           ( \bar{\w}-\w^{(k)} \\
			=&\langle\nabla
                           Q(\w^{(k)})_{\bar{G}{{\setminus}}G^{(k)}},({\bar{\w}-\w^{(k)}})_{\bar{G}{{\setminus}}G^{(k)}}\rangle+\frac{\varphi_{-}(s)}{2}{\|\bar{\w}-\w^{(k)}\|^{2}}
  \\
\geqslant&-\|\nabla Q(\w^{(k)})\|_{G,\infty}\|\bar{\w}-\w^{(k)}\|_{G,1}+\frac{\varphi_{-}(s)}{2}{\|\bar{\w}-\w^{(k)}\|^{2}} \\
\geqslant&-\sqrt{2\varphi_{+}(1)\delta|\bar{G}{{\setminus}}G^{(k)}|}\|\bar{\w}-\w^{(k)}\|+\frac{\varphi_{-}(s)}{2}{\|\bar{\w}-\w^{(k)}\|^{2}}
          \textrm{ (Cauchy-Schwartz inequality)}\\
\geqslant&-\frac{\varphi_{+}(1)\delta
           |\bar{G}{{\setminus}}G^{(k)}|}{\varphi_{-}(s)} \geqslant -\frac{2 \varphi_{+}(1)\delta}{\varphi_{-}(s)}|\{g\in \bar{G} {\setminus} G^{(k)} \ \ | \ \ \| \bar{\w}_{g}\|<\gamma\}|.
\end{align*}

Lastly, we bound the error on selection groups. And this is quite easy if 
we use Lemma \ref{backward}	and the first error bound:
\begin{equation*}
\frac{\delta^{(k)}}{\varphi_{+}(1)}|G^{(k)} {\setminus} \bar{G}| \leqslant \|(\w^{(k)}-\bar{\w})_{G^{(k)} {\setminus} \bar{G}}\|^2
\leqslant \frac{16\varphi_{+}(1)\delta}{\varphi_{-}^{2}(s)}|\{g\in \bar{G} {\setminus} G^{(k)} \ \ | \ \ \| \bar{\w}_{g}\|<\gamma\}|.
\end{equation*}
Again, due to our requirement  in the algorithm, $\delta^{(k)}\geqslant \lambda \delta$, we then have the bound on the error of group selection as in the lemma. 
This  completes the proof of the lemma.
\end{proof}

\begin{theorem} \label{thm:main2}
  \label{termination}
If we take $\delta>\frac{4\varphi_{+}(1)}{\lambda\varphi_{-}^{2}(s)}\|\nabla Q(\bar\w)\|_{G,\infty}^{2}$ and 
require $s$ to satisfy that $s>\bar{k}+(\bar{k}+1)(\sqrt{\kappa(s)}+1)^2(\sqrt{2}\kappa(s))^2\frac{1}{\lambda}$, where
$\bar{k}$ is the number of groups corresponding to $\bar G$, then our algorithm will terminate at some $k\leqslant s-\bar{k}$.
\end{theorem}

\begin{proof}[Proof of Theorem~\ref{thm:main2}]
Suppose our algorithm terminates at some number larger than $s-\bar{k}$, then we may 
just assume the first time $k$ such that  $k=s-\bar{k}$.   Again, we denote the 
supporting groups to be $G^{(k)}$ with the optimal point as $\w^{(k)}$.   Then we 
combine $\bar{G}$  and $G^{(k)}$ together to get 
$G^{\prime}=\bar{G}\cup G^{(k)}$.  Denote 
$\w^{\prime}=\underset{\text{supp}(\w)\subset F^{\prime}}{\text{argmin}}Q(\w)$, where $F^{\prime}=F_{G^{\prime}}$ .   
By requirement of $s$, we have 
$|G^{\prime}|\leqslant s $.

Next, we consider the last step to get $k$ and we have
\begin{align*}
  \delta^{(k)}\geqslant  & \lambda \bigl(Q(\w^{(k-1)})-\underset{g\notin G^{(k-1)},\bd\alpha}{\min}Q(\w^{(k-1)}+\mathbf{E}_{g}\bd\alpha)\bigr) \\
\geqslant&\frac{\varphi_{-}(s)(Q(\w^{(k-1)})-Q(\w^{\prime}))\lambda^2}{\varphi_{+}(1)|G^{\prime}{{\setminus}}G^{(k-1)}|\lambda}\qquad (\textrm{Lemma \ref{forward_process_obj} with } \hat{G}=G^{(k-1)} )\\
       \geqslant&\frac{\varphi_{-}(s)(Q(\w^{(k)})-Q(\w^{\prime}))\lambda}{\varphi_{+}(1)|G^{\prime}{{\setminus}}G^{(k-1)}|}\qquad
                  (\textrm{assumption
         based on}\\
& \qquad \textrm{algorithm and Lemma }\ref{decrease})\\
       \geqslant&\frac{\varphi_{-}(s)\left(\langle\nabla Q(\w^{\prime}),\w^{(k)}-\w^{\prime}\rangle+\frac{\varphi_{-}(s)}{2}\|
       \w^{(k)}-\w^{\prime}\|^{2}\right)\lambda}{\varphi_{+}(1)|G^{\prime}{{\setminus}}G^{(k-1)}|}\qquad
                  (\textrm{definition of }\varphi_{-}(s))\\
       \geqslant&\frac{\varphi_{-}^{2}(s)\|\w^{(k)}-\w^{\prime}\|^{2}\lambda}{2\varphi_{+}(1)|G^{\prime}{{\setminus}}G^{(k-1)}|}.\qquad (\w^{\prime} \textrm{ is optimal on } F^{\prime})
\end{align*}
Now using the result in Lemma \ref{backward} ($\delta^{(k)}\leqslant \frac{\|\w^{(k)}-\bar{\w}\|^{2}\varphi_{+}(1)}{|G^{(k)} {\setminus} \bar{G}|}$)
and the fact that $|G^{\prime} {\setminus} \bar{G}|=|G^{(k)} {\setminus} \bar{G}|$, we have the 
following relation between the two estimations
\begin{align*}
  \|\w^{(k)}-\bar{\w}\|^2\geqslant \left(\frac{\varphi_{-}(s)}{\sqrt{2}\varphi_{+}(1)}\right)^{2}\left(\frac{\lambda|G^{\prime} {\setminus} \bar{G}|}{|G^{\prime}{{\setminus}}G^{(k-1)}|}\right)\|\w^{(k)}-\w^{\prime}\|^2.
\end{align*}

To simplify our notation we introduce a constant
$c=\left(\frac{\varphi_{-}(s)}{\sqrt{2}\varphi_{+}(1)}\right)\left(\sqrt{\frac{\lambda|G^{\prime}
      {\setminus}
      \bar{G}|}{|G^{\prime}{{\setminus}}G^{(k-1)}|}}\right).$ As
$k=s-\bar{k}$ and $s\geqslant 2\bar{k}+1$, a finer calculation and our requirement for the number $s$ give us the following relation
\begin{align*}
  c&=\left(\frac{\varphi_{-}(s)}{\sqrt{2}\varphi_{+}(1)}\right)\left(\sqrt{\frac{\lambda|G^{(k)}{\setminus}
     G^{(k)}\cap\bar{G}|}{|G^{\prime}{{\setminus}}G^{(k)}|+1}}\right) =\left(\frac{\varphi_{-}(s)}{\sqrt{2}\varphi_{+}(1)}\right)\left(\sqrt{\frac{\lambda|G^{(k)} {\setminus} G^{(k)}\cap\bar{G}|}{|\bar{G} {\setminus} G^{(k)}\cap \bar{G}|+1}}\right)\\
&=\left(\frac{\varphi_{-}(s)}{\sqrt{2}\varphi_{+}(1)}\right)\left(\sqrt{\lambda\frac{|G^{k}|-|G^{(k)}\cap\bar{G}|}{|\bar{G}|-|G^{(k)}\cap
  \bar{G}|+1}}\right)
  =\left(\frac{\varphi_{-}(s)}{\sqrt{2}\varphi_{+}(1)}\right)\left(\sqrt{\lambda\frac{k-|G^{(k)}\cap\bar{G}|}{\bar{k}-|G^{(k)}\cap
  \bar{G}|+1}}\right) \\
&\geqslant
  \left(\frac{\varphi_{-}(s)}{\sqrt{2}\varphi_{+}(1)}\right)\left(\sqrt{\lambda\frac{s-\bar{k}
  }{\bar{k}+1}}\right) \geqslant  \sqrt{\kappa(s)}+1. \qquad (\textrm{due to our definition of } s)
\end{align*}
So we have 
\begin{align}\label{triangle}
  c\|\w^{(k)}-\w^{\prime}\|\leqslant\|\w^{(k)}-\bar{\w}\|\leqslant \|\w^{(k)}-\w^{\prime}\|+\|\w^{\prime}-\bar{\w}\|,
\end{align}
 which implies that $(c-1)\|\w^{(k)}-\w^{\prime}\| \leqslant
 \|\w^{\prime}-\bar{\w}\|$. The following calculation, together with (\ref{triangle}), shows that $Q(\w^{(k+1)})\leqslant Q(\bar{\w})$:
\begin{align*}
 & Q(\w^{(k)})-Q(\bar{\w}) =Q(\w^{(k)})-Q(\w^{\prime})-(Q(\bar{\w})-Q(\w^{\prime}))\\
\leqslant & \frac{\varphi_{+}(s)}{2}\|\w^{(k)}-\w^{\prime}\|^{2}-\frac{\varphi_{-}(s)}{2}\|\w^{\prime}-\bar{\w}\|^{2}
\leqslant
            \left(\frac{\varphi_{+}(s)}{2(c-1)^2}-\frac{\varphi_{-}(s)}{2}\right)\|\w^{\prime}-\bar{\w}\|^{2} \leqslant 0.
\end{align*}
However, this contradicts  (\ref{eq:decrease}). So the assumption is incorrect and it completes the proof.
\end{proof}

Theorem~\ref{thm:main_giga} follows directly from statements of
Theorem~\ref{errorestimation_gdt} and Theorem~\ref{convergency_gdt} in
the following. We can see that conclusions in
Theorems~\ref{errorestimation_gdt} and \ref{convergency_gdt} for GIGA
are similar to those of
Theorems~\ref{estigood} and \ref{thm:main2} for IGA, 

\begin{theorem}
\label{errorestimation_gdt}
Let  $\varepsilon$ in GIGA algorithm satisfy
$\varepsilon \geqslant \frac{2\sqrt{2}{\varphi}_{+}(1)\|\nabla Q(\bar{\w})\|_{G,\infty}}{\lambda{\varphi}_{-}(s)}$.
Suppose the algorithm stops at $k$, and the supporting group is $G^{(k)}$ with supporting features as $F^{(k)}=F_{G^{(k)}}$. 
Then the GIGA estimator will satisfy
\begin{align}
\|\bar{\w}-\w^{(k)}\|^2  &\leqslant \frac{8\varepsilon^{2}\bar{\Delta}}{{\varphi}^{2}_{-}(s)}\\
Q(\w^{(k)})-Q(\bar{\w})&\leqslant 
 \frac{ {\varepsilon}^{2}\bar{\Delta}}{{\varphi}_{-}(s)} 
 \\
|\bar{G}\setminus G^{(k)}| & \leq 2\bar{\Delta}
\\
|G^{(k)}\setminus \bar{G}| &\leqslant \frac{16{\varphi}^{2}_{+}(1)\bar{\Delta}}{\lambda^2 {\varphi}^{2}_{-}(s)}.
\end{align}
where $\bar{\Delta}=| \{g \in \bar{G}\setminus G^{(k)} \ \ : \ \ \| \bar{\w}_{g}\|<\gamma \}|$ and  $\gamma=\frac{2\sqrt{2}\varepsilon}{{\varphi}_{-}(s)}$.
\end{theorem}

\begin{proof}[Proof of Theorem~\ref{errorestimation_gdt}]
By our previous lemmas, if $\varepsilon \geqslant \frac{2\sqrt{2}{\varphi}_{+}(1)\|\nabla Q(\bar{\w})\|_{G,\infty}}{\lambda{\varphi}_{-}(s)}$, we will have 
\begin{align*}
  0 & \geqslant  Q(\bar{\w})-Q(\w^{(k)}) & \\
    & \geqslant  \langle \nabla Q(\w^{(k)}), \bar{\w}-\w^{(k)} \rangle +   \frac{{\varphi}_{-}(s)\|\bar{\w}-\w^{(k)}\|^{2}}{2}. & 
\end{align*}
A simple calculation gives the following relations:
\begin{align*}
  \frac{{\varphi}_{-}(s)\|\bar{\w}-\w^{(k)}\|^{2}}{2}& \leqslant -\langle 
  \nabla Q(\w^{(k)}), \bar{\w}-\w^{(k)}\rangle & \\
  & \leqslant\|\nabla Q(\w^{(k)})\|_{G,\infty} 
  \sqrt{|\bar{G}\setminus G^{(k)}|} \|\bar{\w}-\w^{(k)}\| &\\
    & \leqslant \varepsilon
  \sqrt{|\bar{G}\setminus G^{(k)}|} \|\bar{\w}-\w^{(k)}\|.& 
\end{align*}
So we have 
\begin{equation}
\label{gdt_estimation_error}
\|\bar{\w}-\w^{(k)}\| \leqslant \frac{2\varepsilon
  \sqrt{|\bar{G}\setminus G^{(k)}|}}{{\varphi}_{-}(s)}.
  \end{equation}
  Here we point out that when the algorithm terminates, the $\ell_{G,\infty}$ norm is less than $\varepsilon$ not  $\lambda \varepsilon$ according to our algorithm.
Writing in another form, we have 
\begin{align*}
  \frac{4\varepsilon^{2}|\bar{G}\setminus G^{(k)}|}{{\varphi}^{2}_{-}(s)}& 
  \geqslant \|\bar{\w}-\w^{(k)}\|^{2}\\
  &\geqslant \|(\bar{\w}-\w^{(k)})_{\bar{G}\setminus G^{(k)}}\|^{2}=\|\bar{\w}_{\bar{G}\setminus G^{(k)}}\|^{2}\\
  &\geqslant \gamma^{2} | \{ g \in \bar{G} \setminus G^{(k)} \ \ | \ \ \| \bar{\w}_{g}\|\geqslant 
\gamma\} |. 
\end{align*}
If we let 
$\gamma^{2}=\frac{8\varepsilon^{2}}{{\varphi}^{2}_{-}(s)}$, we can get
a new estimation using the same trick as in IGA:
\begin{equation*}
  |\bar{G}\setminus G^{(k)} | \leqslant 2 | \{g \in \bar{G} \setminus G^{(k)} \ \ | \ \ \| \bar{\w}_{g}\|<\gamma \}| =2\bar{\Delta}.
\end{equation*}
From above, we have
\begin{align*}
  & Q(\bar{\w})-Q(\w^{(k)})&\\
 \geqslant  &\langle \nabla Q(\w^{(k)}), \bar{\w}-\w^{(k)} \rangle +   \frac{{\varphi}_{-}(s)\|\bar{\w}-\w^{(k)}\|^{2}}{2}& \\
 =  &\langle \nabla Q(\w^{(k)})_{\bar{G}\setminus G^{(k)}}, {\bar{\w}-\w^{(k)}}_{\bar{G}\setminus G^{(k)}} \rangle +   \frac{{\varphi}_{-}(s)\|\bar{\w}-\w^{(k)}\|^{2}}{2} & \\
  \geqslant& -\|\nabla Q(\w^{(k)})\|_{G,\infty} 
   \|\bar{\w}-\w^{(k)}\|_{G,1}+   \frac{{\varphi}_{-}(s)\|\bar{\w}-\w^{(k)}\|^{2}}{2} & \\
    \geqslant& -\|\nabla Q(\w^{(k)})\|_{G,\infty} 
  \sqrt{|\bar{G}\setminus G^{(k)}|} \|\bar{\w}-\w^{(k)}\|_{G,2}+   \frac{{\varphi}_{-}(s)\|\bar{\w}-\w^{(k)}\|^{2}}{2} & (\textrm{Cauchy-Schwartz})\\
  \geqslant& -\varepsilon 
  \sqrt{|\bar{G}\setminus G^{(k)}|} \|\bar{\w}-\w^{(k)}\|+   \frac{{\varphi}_{-}(s)\|\bar{\w}-\w^{(k)}\|^{2}}{2}&  \\
  \geqslant & -\frac{|\bar{G}\setminus G^{(k)}| {\varepsilon}^{2}}{2{\varphi}_{-}(s)}. &
\end{align*}
So we can bound the difference of $Q$ as 
\begin{equation*}
Q(\w^{(k)})-Q(\bar{\w})\leqslant \frac{|\bar{G}\setminus G^{(k)}| {\varepsilon}^{2}}{2{\varphi}_{-}(s)}\leqslant 
 \frac{\bar{\Delta} {\varepsilon}^{2}}{{\varphi}_{-}(s)}.
\end{equation*}
And the error of parameter estimation (\ref{gdt_estimation_error}) can be simplified to
\begin{equation*}
\|\bar{\w}-\w^{(k)}\| \leqslant \sqrt{\frac{4\varepsilon^{2}|\bar{G}\setminus G^{(k)}|}{{\varphi}^{2}_{-}(s)}}\leqslant 
  \sqrt{\frac{8\varepsilon^{2}\bar{\Delta}}{{\varphi}^{2}_{-}(s)}}.
\end{equation*}
For the group selection, we have the following relation due to Lemma \ref{backwardgdt} and (\ref{deltabound})
\begin{equation*}
|G^{(k)}\setminus  \bar{G} | \leqslant \frac{2{\varphi}^{2}_{+}(1)\|\bar{\w}-\w^{(k)}\|^2}{\lambda^2 \varepsilon^2} \leqslant \frac{16{\varphi}^{2}_{+}(1)\bar{\Delta}}{\lambda^2 {\varphi}^{2}_{-}(s)}.
\end{equation*}
 It completes the proof of Theorem~\ref{errorestimation_gdt}.
\end{proof}

We need our algorithm to be terminated in as few steps as possible.   If we 
denote $k$ as the iteration that the algorithm stops, then we will show this 
$k\leqslant s-\bar{k}$.   Here, $s$ is the group sparsity and $\bar{k}$ is the 
number of supporting groups for the optimal solution. 
\begin{theorem}
  \label{convergency_gdt}
  If $\varepsilon \geqslant \frac{2\sqrt{2}{\varphi}_{+}(1)\|\nabla Q(\bar{\w})\|_{G,\infty}}{\lambda{\varphi}_{-}(s)}$ and we
  require $s$ to a number such that
  $s>\bar{k}+(\bar{k}+1)(\sqrt{\kappa(s)}+1)^2(2\kappa(s))^2\frac{1}{\lambda}$, where
$\bar{k}$ is the number of groups corresponding to the global optimal sparse solution.
Then our algorithm will terminate at some $k\leqslant s-\bar{k}$.
\end{theorem}
 The proof of Theorem~\ref{convergency_gdt} is similar to that of Theorem \ref{termination} and we thus
 omit its details.

\section{Proofs of Theorems~\ref{probability_hypo_sp} and \ref{thm:linear}}
\label{subsec:proof2}

\begin{proof}[Proof of Theorem~\ref{probability_hypo_sp}]
Consider the following notation for matrix. For a matrix $X$ in
$\R^{n\times p}$, $X_g$ will be a $\R^{n\times |g|}$ matrix that only
keep the columns corresponding to $g$, where $g$ is an index set in $\{1,2,\cdots,p\}$. We denote $\Sigma_g=X_g^{\top}X_g$,
For the theorem, we can first show that 
  $\|X_g^{\top}\varepsilon\|\leqslant \sqrt{n}\left(\sqrt{|g|}+\sqrt{2\varphi_{+}^{2}(1)\log(\frac{1}{\eta})}\right)$
  with probability $1-\eta$.
To this end, we have to point out that our columns of $X$ are normalized to $\sqrt{n}$ and hence
$X_{g}^{\top}\varepsilon$  will be a $\frac{p}{m}$-variate Gaussian random variable with $n$
 on the diagonal of covariance  matrix. 
 We further use $\lambda_i$ as  the eigenvalues of $\Sigma_g$ with
 decreasing order. Using that $tr(\Sigma_g^2)=\lambda^2_1+\lambda^2_2+\cdots+\lambda^2_{|g|}$ and Proposition 1 of \citet{hsu2012tail}, we have
\begin{align*}
& e^{-t}\geqslant \textbf{Pr}\Bigl(||X_g^{\top}\varepsilon||^{2}>\sum_{i=1}^{|g|}\lambda_{i}+2\sqrt{\sum_{i=1}^{|g|}\lambda_{i}^{2}t}+2\lambda_{1}t\Bigr)\\
 \geqslant & \textbf{Pr}\Bigl(||X_g^{\top}\varepsilon||^{2}>\sum_{i=1}^{|g|}\lambda_{i}+2\sqrt{2\sum_{i=1}^{|g|}\lambda_{i}\lambda_1 t}+2\lambda_{1}t\Bigr) \geqslant \textbf{Pr}\Bigl(||X_g^{\top}\varepsilon||\geqslant\sqrt{\sum_{i=1}^{|g|}\lambda_{i}}+\sqrt{2\lambda_{1}t}\Bigr).
\end{align*} 
Substitute $t$ with $\log(\frac{1}{\eta})$ and the facts that
$\sum_{i=1}^{|g|}\lambda_{i}=|g|n$ and
$\lambda_1=\|\Sigma_g\|\leqslant n \varphi_{+}(1)$, we have
$\|X_g^{\top}\varepsilon\|\leqslant
\sqrt{n}\left(\sqrt{|g|}+\sqrt{2\varphi_{+}(1)\log(\frac{1}{\eta})}\right)$
with probability $1-\eta$. If we replace $\eta$ as $\frac{\eta}{m}$ for each group $g$ and take the maximal of all group norms, we 
extend the result to the general case:
 \begin{equation}
\label{eq:epsguassian}
\underset{g\in G}{\operatorname{max}}\
\|X_g^{\top}\varepsilon\|\leqslant \sqrt{n}\left(\sqrt{\underset{g\in
      G}{\operatorname{max}}|g|}+\sqrt{2\varphi_{+}(1)\log(\frac{m}{\eta})}\right)
\end{equation}
  with probability $1-{\eta}$. Recall our general convex function has a factor of $\frac{1}{2n}$,  we have
\[\|\nabla Q(\w^{\ast})\|_{G,\infty}\leqslant \frac{1}{\sqrt{n}}\left(\sqrt{\underset{g\in G}{\operatorname{max}}{|g|}}+\sqrt{2\varphi_{+}(1)\log(\frac{m}{\eta})}\right)\]
  with high  probability. The major difference between the true solution and optimal solution supported on $S$ with $|S|\leqslant \bar{k}$ is that we need to estimate $X_g^{\top}\left( X_S(X_S^{\top}X_S)^{-1}X_S^{\top}-I_n \right)\varepsilon$ instead of $X_g^{\top}\varepsilon$ for each group $g$. In fact, we are considering $I_n-X_S(X_S^{\top}X_S)^{-1}X_S^{\top}$ since the noise itself makes this negative sign not important.  Note, we can decompose $I_n-X_S(X_S^{\top}X_S)^{-1}X_S^{\top}=P_g P_g^{\top}$, where $P_g$ has dimension $n \times (n-|F_S|)$ and $P_g^{\top} P_g=I_{n-|F_S|}$, due to the property of the projection matrix $X_S(X_S^{\top}X_S)^{-1}X_S^{\top}$.
 That is, we transform our question into estimating $X_g^{\top}P_g \varepsilon^{\prime}$, where $\varepsilon^{\prime}$ is a
  vector of i.i.d. $\mathcal{N}(0,1)$ random variables of dimension $n-|F_S|$. We denote the largest eigenvalue of $X_g^{\top}P_gP_g^{\top}X_g$ as $\lambda^{\prime}$ and it's easy to observe that $\lambda^{\prime}\leqslant\lambda$. Also the trace of $X_g^{\top}P_gP_g^{\top}X_g$ is no big than the trace of $X_g^{\top}X_g$, so we easily extend the above result to the optimal value and we  have
        $\|\nabla Q(\bar{\w})\|_{G,\infty}\leqslant \frac{1}{\sqrt{n}}\left(\sqrt{\underset{g\in G}{\operatorname{max}}{|g|}}+\sqrt{2\varphi_{+}(1)\log(m)}\right)$
      with high probability. The second statement follows by arguments
      similar to that of \eqref{eq:epsguassian}. This completes the proof.
\end{proof}

\begin{proof}[Proof of Theorem~\ref{thm:linear}]
In light of Theorem~\ref{thm:main}, we intend to show $\bar{\Delta}
\leq \Delta_{\text{IGA}}$ when two constants ``$\Omega$'s'' are
appropriately chosen. It suffices to show that if  $\|
\bar{\w}_{g}\|<\Omega (\|\nabla Q(\bar{\w})\|_{G, \infty})=
\Omega\left(\sqrt{ {p\over mn} +{\log (m)\over n}}\right)$ holds, then
$\|\w^*_g\| \leq \Omega\left(\sqrt{ {p\over mn} +{\log (m)\over
      n}}\right)$; Indeed, by the second statement of
Theorem~\ref{probability_hypo_sp},  $\|\w_g^*\| \leqslant \|\bar{\w}_g\|+\|(\bar{\w}-\w^*)_g\| \leq  \Omega\left( \sqrt{\frac{{p\over m}+\log m}{n}}\right)$.
Then, combining the second statement in Theorem~\ref{thm:main},
$\bar{\Delta} \leq \Delta_{\text{IGA}}$ and the fact that
$\norm{\w^*-\bar \w}=O_p(k_*/n)$, we obtain the
first claim. Combining the last claim in Theorem~\ref{thm:main} with
$\bar{\Delta} \leq \Delta_{\text{IGA}}$ implies the second claim. The
proof of Theorem~\ref{thm:linear} is complete.
\end{proof}

\section{Proofs of Theorems~\ref{thm:logistic_oracle} and \ref{thm:logistic}}
\label{subsec:proof3}

\begin{proof}[Proof of Theorem~\ref{thm:logistic_oracle}]
Define $\tilde\y=(\tilde y_1,\cdots,\tilde y_n)^T$, where $\tilde
y_i=(y_i+1)/2$. We then have $\nabla Q(\w)=X^T(\mathbf p(\w)-\tilde\y)$,
where $\mathbf p(\w)=(p_1(\w),\cdots,p_n(\w))^T$.  Since
$\E(\tilde\y|X)=\mathbf p(\w)$ and  $\tilde y_i$'s are bounded, the upper-bound statement
on gradient group norm $\|\nabla Q(\w^{\ast})\|_{G,\infty}$ follows
from similar arugments as \eqref{eq:epsguassian}.  Next, we intend to
show the upper bound for $\norm{\bar\w-\w^*}_{G,\infty}$. Given $r>0$,
define $r$-radius ball in $\ell_{G,\infty}$ norm as
$B(r)=\{\Delta\in\real^p: \norm{\Delta_{\bar G}}_{\bar G,\infty}\leq
r,\, \Delta_{G\backslash\bar G}=\mathbf 0\}$. Define
function $H: B(r)\subset\real^p\rightarrow \real^p$ to be $H_{\bar G}(\Delta) = (X_{\bar G}^TW^*X_{\bar G})^{-1} X_{\bar
  G}^T(\tilde\y-\mathbf p(w^*+\Delta))+\Delta_{\bar G}$,
and $H_{G\backslash\bar G}(\Delta)=\mathbf 0$. 
First, note that for any small $\eta>0$, there is some $r_n=\Omega(\sqrt{\frac{q+\log m}{n}})$ such that $\mathcal
A_n:=\{H(B(r_n))\subset B(r_n)\}$ holds with probabiity greater than
$1-\eta$ for large enough $n$.  Indeed, Suppose that $\Delta\in
B(r_n)$ and the event  $\mathcal B_n:=\{\norm{X_{\bar G}^TW^*X_{\bar G})^{-1}X_{\bar G}^T(\mathbf
  p(\w^*)-\tilde\y)}_{\bar G,\infty}\leq r_n/2\}$ holds. By Taylor expansion,
there is $\tilde\Delta\in B(r_n)$  with same signs of  $\Delta$ such that
\begin{align}
X_{\bar G}^T\mathbf p(\w^*+\Delta) & = X_{\bar G}^T\mathbf p(\w^*) +
X_{\bar G}^T W^* X_{\bar G}\Delta_{\bar G}+X_{\bar
G}^T\bigl(W(\w^*+\tilde\Delta)-W^*\bigr)X_{\bar G}\Delta_{\bar G}\notag\\
&=: X_{\bar G}^T\mathbf p(\w^*) +
X_{\bar G}^T W^* X_{\bar G}\Delta_{\bar G} + R_{\bar G}(\tilde\Delta), \label{eq:taylor}
\end{align}
where $R(\tilde\Delta)=X^T\bigl(W(\w^*+\tilde\Delta)-W^*\bigr)X_{\bar
  G}\Delta_{\bar G}$. Consequently, by the mean value theorem and
boundedness of covariate domain and $U_1$, there is a  positive constant
$C_{11}$ such that
\begin{equation}
\label{eq:Rdelta}
\norm{R(\tilde\Delta)}_{G,\infty} \leq \max_{g\in G}
\sum_{i=1}^n\norm{\x_{ig}} \tilde\Delta^T\x_i^T\x_i\Delta\leq
nC_{11}\bar k^{1/2}\norm{\Delta}_{G,\infty}^2\leq nC_{11}\bar k^{1/2}r_n^2.
\end{equation}
The two displays above together with boundedness of $U_2$  imply that
\begin{align*}
&H_{\bar G}(\Delta) = (X_{\bar G}^TW^*X_{\bar G})^{-1}\big( X_{\bar
  G}^T(\y-\mathbf p(\w^*))-R_{\bar G}(\tilde\Delta) \bigr),\\
&\norm{H_{\bar G}(\Delta)}_{\bar G,\infty} \leq r_n/2+ \norm{(X_{\bar
    G}^TW^*X_{\bar G})^{-1}R(\tilde\Delta)}_{\bar G,\infty}\leq r_n/2+
C_{11}\bar k r_n^2.
\end{align*}
In addition, by \cite{hsu2012tail}, we can obtain that  $P(\mathcal B_n)>1-\eta$ with some
$r_n=\Omega(\sqrt{\frac{q+\log m}{n}})$. In together with previous
display and our choice of $r_n$, we obtain that $P(\mathcal
A_n)>1-\eta$ for large enough $n$. 

Next, suppose $\mathcal A_n$ holds. By Brouwer's fixed-point theorem,
there exists a point $\hat\Delta\in B(r_n)$ such that
$H(\hat\Delta)=\hat\Delta$, which implies that $X_{\bar
  G}^T(\tilde\y-\mathbf p(\w^*+\hat\Delta))=0$ and
$\hat\Delta_{G\backslash\bar G}=\mathbf 0$. By uniqueness of $\bar\w$,
we have $\Delta=\bar\w-\w^*$, and therefore,
$\norm{\bar\w-\w^*}_{G,\infty}=\norm{\hat\Delta}_{G,\infty}\leq r_n$.  

It remains to show the upper bound statement
on $\|\nabla Q(\bar\w)\|_{G,\infty}$. By arguments similar to that of
\eqref{eq:taylor} and setting $\Delta=\hat\Delta$, we have
\begin{equation*}
\nabla Q(\bar\w) = \nabla Q(\w^*)+\frac{1}{n} X^TW^*X_{\bar
  G}\hat\Delta_{\bar G}+\frac{1}{n}R(\tilde\Delta).
\end{equation*} 
Since $\nabla_{\bar G} Q(\bar\w)=\mathbf 0$, previous display implies
that
\begin{align}
\hat\Delta_{\bar G}&=(\frac{1}{n}X_{\bar G}^TW^*X_{\bar G})^{-1} \bigl(
-\frac{1}{n}X_{\bar G}^T(\mathbf p(\w^*)-\tilde\y)-\frac{1}{n}R_{\bar G}(\tilde\Delta)
\bigr),\notag\\
\nabla_{G\backslash\bar G} Q(\bar\w)&=\nabla_{G\backslash\bar G}
Q(\w^*)  -\frac{1}{n}\Theta X_{\bar G}^T(\mathbf p(\w^*)-\tilde\y) -
\frac{1}{n}\Theta R_{\bar G}(\tilde\Delta)+\frac{1}{n}R_{G\backslash\bar G}(\tilde\Delta), \label{eq:logistic_last}
\end{align} 
where  $\Theta=\frac{1}{n}X_{G\backslash\bar G}^TW^*X_{\bar G}
(\frac{1}{n}X_{\bar G}^TW^*X_{\bar G})^{-1}$. Since $U_2$ and $U_3$
are bounded, following arguments like \eqref{eq:epsguassian}, the
second term in \eqref{eq:logistic_last} is upper bounded by $O_p(\sqrt{\frac{q+\log
    m}{n}})$ in $l_{G,\infty}$ norm as the first term; following arguments like
\eqref{eq:Rdelta}, the third and fourth terms are $O_p(\bar k
r_n^2)$. Therefore, $\norm{\nabla Q(\bar\w)}_{G,\infty}\leq O_p(\sqrt{\frac{q+\log
    m}{n}})$. We complete the proof of
Theorem~\ref{thm:logistic_oracle}.
\end{proof}

With Theorem~\ref{thm:logistic_oracle}, Theorem~\ref{thm:logistic} can
be proved by similar procedures as that of
Theorem~\ref{thm:linear}.

\section{Lemmas and Proofs}
\label{sec:suppthm}
In the following, Lemma~\ref{lem:cond} connects Assumption \ref{ass:s}
to a natural random design
scenario. Lemmas~\ref{lem:lemforward}--\ref{forward_process_obj} are
useful intermediate results for our analysis of the IGA algorithm, and
Lemmas~\ref{lem:giga11}--\ref{boundofiteration_gdt} are useful for the GIGA algorithm.

\begin{lemma} 
\label{lem:cond}
Let $Q(\w) = {1\over n}\sum_{i=1}^n f_i(\x_i^\top\w)$. All data points
$\x_i$'s follow from i.i.d. isotropic sub-Gaussian distribution with
enough samples $n \geq \Omega(k_{\bar{k}} + \bar{k}\log m)$. Let
$L:=\sup\limits_{z\in \mathcal{Z}}(\nabla^2 f_i(z))$ and
$l=\inf\limits_{z\in \mathcal{Z}}(\nabla^2 f_i(z))$ where
$\mathcal{Z}:=\{\x_i^\top\w:~\w\in \mathcal{D}, \forall i=1,
2,\cdots\}$.  The condition number $\tau$ is defined as $\tau:={L\over
  l}$. If the condition number $\tau$ is bounded, then with high probability there exists an $s$ satisfying Assumption~\ref{ass:s}.
\end{lemma}

\begin{proof}[Proof to Lemma~\ref{lem:cond}]
Let the upper bound of the maximal eigenvalue and the lower bound of
the minimal eigenvalue of $\nabla^2 f(z)$ be $L$ and $l$
respectively. The condition number $\tau$ is defined as $\tau =
{L/l}$. With slight abuse of notation, let $g$ be an index set in
$\{1,\cdots,p\}$. By \eqref{eq:def_rho+} and the convexity of $f_i(\cdot)$, we have
\begin{align}
\nonumber
n\rho_+(g) = & \lambda_{\max}\left(\sum_{i=1}^n(\x_i)_g(\x_i)_g^\top \nabla^2 f_i(\x_i^\top \w)\right) 
\\
\nonumber
 \leq & \max_{z, i}\nabla^2 f_i(z) \lambda_{\max} \left(\sum_{i=1}^n (\x_i)_g (\x_i)_g^\top\right)
 \\
 \leq & L \lambda_{\max}\left(\sum_{i=1}^n (\x_i)_g (\x_i)_g^\top\right).
 \label{eq:proof_cond_1}
\end{align}
Similarly, for $\rho_-(g)$, we have
\begin{align*}
n\rho_-(g) \geq l \lambda_{\min}\left(\sum_{i=1}^n (\x_i)_g (\x_i)_g^\top\right).
\end{align*}
Let $T_s$ be a super set: $T_s = \{F_S~|~|S|\leq s, S\subset G\}$ and $k_s := \max\limits_{g\in T_s} |g|$. From the random matrix theory \cite[Theorem~5.39]{vershynin2010introduction}, for any $g\in T_s$ we have
\begin{align}
\sqrt{\lambda_{\max}\left(\sum_{i=1}^n (\x_i)_g (\x_i)_g^\top\right)} \geq \sqrt{n}+\Omega(\sqrt{|g|})+\Omega(t)
\label{eq:proof_cond_2}
\end{align}
holds with probability at most $\exp\{-t^2\}$. 
Then we use the union bound to provide an upper probability bound for $\varphi_+(s)$
\begin{align}
\nonumber
&\P\left(\sqrt{n\varphi_+(s)/L} \geq \sqrt{n}+\Omega(\sqrt{k_s})+\Omega(t)\right)&
\\
\nonumber
\leq & \sum_{g\in T_s}\P\left(\sqrt{n\varphi_+(s)/L} \geq \sqrt{n}+\Omega(\sqrt{k_s})+\Omega(t)\right)&
\\
\nonumber
\leq & \sum_{g\in T_s}\P\left(\sqrt{n\varphi_+(s)/L} \geq \sqrt{n}+\Omega(\sqrt{|g|})+\Omega(t)\right)&
\\
\nonumber 
\leq & \left(\begin{matrix} m \\ s \end{matrix}\right)\exp\{-t^2\}\quad &(\text{from \eqref{eq:proof_cond_1} and \eqref{eq:proof_cond_2}})
\\
\nonumber 
\leq & \exp\{s\log (m)-t^2\}.
\end{align}
Taking $t=\Omega(\sqrt{s \log m})$, it follows
\begin{align}
\P\left(\sqrt{n\varphi_+(s)/L} \geq \sqrt{n}+\Omega(\sqrt{k_s}+\sqrt{s \log m})\right) \leq \exp\{-\Omega(s \log m)\}.
\end{align}
Similarly, we have
\begin{align}
\P\left(\sqrt{n\varphi_-(s)/l} \geq \sqrt{n}-\Omega(\sqrt{k_s}+\sqrt{s \log m})\right) \leq \exp\{-\Omega(s \log m)\}.
\end{align}
Now we are ready to bound $\tau(s)$. We can choose the value of $n$ large enough such that $\sqrt{n}$ is greater than $3\Omega(\sqrt{k_s}+\sqrt{s \log m})$, which indicates
\[
\P\left(\sqrt{n\varphi_+(s)/L} \leq {{4\over 3}\sqrt{n}}\right) \geq 1-\exp\{-\Omega(s \log m)\}
\]
and
\[
\P\left(\sqrt{n\varphi_-(s)/l} \geq {{2\over 3}\sqrt{n}}\right) \geq 1-\exp\{-\Omega(s \log m)\}
\]
It follows that
\begin{align}
\P(\kappa(s) \leq {2}\tau) = \P\left(\frac{\sqrt{n\varphi_+(s)/L}}{\sqrt{n\varphi_-(s)/l}} \leq {\sqrt{2}}\right) \geq 1- 2\exp\{-\Omega(s \log m)\}.
\end{align}
Since $\tau$ is bounded, there exists an $s=\Omega (\bar{k})$ satisfying Assumption~\ref{ass:s}. The required number of samples is $n \geq \Omega(k_s + s\log m) = \Omega(k_{\Omega(\bar{k})} + \Omega(\bar{k})\log m) = \Omega(k_{\bar{k}}+ \bar{k}\log m)$.
It completes the proof.
\end{proof}

For analysis of the IGA algorithm, we transform the difference of criteria
functions into the  norms of the partial derivatives of $Q$.
\begin{lemma}
\label{lem:lemforward}
  For any group $g \in G$ and $\w \in \mathcal{D}$,  if 
  \[ Q(\w)- \underset{\bd\alpha \in \mathbb{R}^{|g|}}{\operatorname{min}} Q(\w+\mathbf{E}_{g}\bd\alpha)\geqslant \lambda \delta, \]
  then
  \[\|\nabla_{g} Q(\w)\| \geqslant \sqrt{2\varphi_{-}(1)\lambda\delta}.\]
\end{lemma}

\begin{proof}[Proof of Lemma~\ref{lem:lemforward}]
 Note that we specify the dimension of $\bd\alpha$ explicitly here, but it is quite trivial to find the correct dimension and hence we will ignore the specific dimension
from now on for the simplicity of calculation. Taking the parameter $\lambda$ into considerations, we have 
  \begin{align*}
  -\lambda \delta  &\geqslant \underset{\bd\alpha}{\operatorname{min}} \ Q(\w+\mathbf{E}_{g}\bd\alpha)-Q(\w)&\\
 {} &\geqslant \underset{\bd\alpha}{\operatorname{min}}\ \langle\nabla Q(\w),\mathbf{E}_{g}\bd\alpha\rangle+\frac{\varphi_{-}(1)}{2}\|\mathbf{E}_{g}\bd\alpha \|^{2}& (\textrm{from the definition of $\varphi_{-}(1)$})\\
  {} &= \underset{\bd\alpha}{\operatorname{min}}\ \langle\nabla_{g} Q(\w), \bd\alpha\rangle+\frac{\varphi_{-}(1)}{2}\| \bd\alpha \|^{2}&\\
  &=-\frac{\|\nabla_{g} Q(\w)\|^{2}}{2\varphi_{-}(1)}.&
  \end{align*}

  Hence we have $\|\nabla_{g} Q(\w)\| \geqslant \sqrt{2\varphi_{-}(1)\lambda\delta} $ and this   completes the proof of the lemma.
  \end{proof}

Similarly, we can have a conclusion in the other direction.

\begin{lemma}\label{lemma2}
  For any $\w \in  \mathcal{D}$, if $\displaystyle{{Q(\w)}-\underset{g \in G, \bd\alpha}{\operatorname{min}}Q(\w+\mathbf{E}_{g}\bd\alpha)}\leqslant \delta$, 
  then  
  \[\|\nabla Q(\w)\|_{G,\infty} \leqslant \sqrt{2\varphi_{+}(1)
\delta}.\]
\end{lemma}
\begin{proof}[Proof of Lemma~\ref{lemma2}]
Note that
   \begin{align*}
  \delta&\geqslant Q(\w)-\underset{g \in G, \bd\alpha}{\operatorname{min}}Q(\w+\mathbf{E}_{g}\bd\alpha)&\\
{}&= \underset{g \in G, \bd\alpha}{\operatorname{max}}-\left(Q(\w+\mathbf{E}_{g}\bd\alpha)-Q(\w)\right)&\\
 {} &\geqslant \underset{g \in G, \bd\alpha}{\operatorname{max}}-\langle\nabla Q(\w),\mathbf{E}_{g}\bd\alpha\rangle-\frac{\varphi_{+}(1)
}{2}\|\mathbf{E}_{g} \bd\alpha\|^{2}&(\textrm{from the definition of $\varphi_{+}(1)$})\\
 {} &= \underset{g \in G, \bd\alpha}{\operatorname{max}}-\langle\nabla_{g} Q(\w), \bd\alpha\rangle-\frac{\varphi_{+}(1)
}{2}\|\bd\alpha\|^{2}&\\
{} &= \underset{g \in G}{\operatorname{max}}\ \frac{\|\nabla_{g} Q(\w)\|^2}{2\varphi_{+}(1)
}&\\
  {}&=\frac{\|\nabla Q(\w)\|^2_{G, \infty}}{2\varphi_{+}(1)
}.&
  \end{align*}
  Taking square roots on both sides, we can find the relation in the lemma. And it completes the proof.
\end{proof}

Due to the stop condition in IGA algorithm, we have the following relations between the group selections and parameter estimations. 
\begin{lemma}\label{backward} 
At the end of each backward step and hence when our algorithm stops, we have the following connections of our selected parameters and groups:
  \begin{equation*}
     \|\w^{(k)}_{G^{(k)} {\setminus} \bar{G}}\|^2=\|(\w^{(k)}-\bar{\w})_{G^{(k)}{\setminus}\bar{G}}\|^2
    \geqslant     \frac{\delta^{(k)}}{\varphi_{+}(1)}|G^{(k)} {\setminus} \bar{G}|\geqslant 
    \frac{\lambda\delta}{\varphi_{+}(1)}|G^{(k)} {\setminus} \bar{G}|.
  \end{equation*}
\end{lemma}

\begin{proof}[Proof of Lemma~\ref{backward}]
Let's consider all candidates of removal groups in the backward step, and we have 
\begin{align*}
   {}&|G^{(k)} {\setminus} \bar{G}| \underset{g \in G^{(k)} {\setminus} \bar{G}}{\operatorname{min}}Q(\w^{(k)} -\w^{(k)}_{g})&\\
  \leqslant &\underset{g \in G^{(k)} {\setminus} \bar{G}}{\sum}Q(\w^{(k)} - \w^{(k)}_{g})&\\
 \leqslant &\underset{g \in G^{(k)} {\setminus} \bar{G} }{\sum}\left( Q(\w^{(k)})-\langle\nabla Q(\w^{(k)}), \w^{(k)}_{g}\rangle\right)
 +\underset{g \in G^{(k)} {\setminus} \bar{G} }{\sum}\frac{\varphi_{+}(1)}{2}\|\w^{(k)}_{g}\|^{2}  \\
= &|G^{(k)} {\setminus} \bar{G}|Q(\w^{(k)})+ \frac{\varphi_{+}(1)}{2}
 \|\w^{(k)}_{{G^{(k)} {\setminus} \bar{G}}}\|^2. &(\w^{(k)} \textrm{ is the optimal on } G^{(k)})\\
\end{align*}

Hence, we have 

\[|G^{(k)} {\setminus} \bar{G}|  \underset{g\in G^{(k)}}{\operatorname{min}}Q\left(\w^{(k)} - 
\w^{(k)}_{g})-Q(\w^{(k)}\right)\leqslant  \frac{\varphi_{+}(1)}{2}\|\w^{(k)}_{G^{(k)}{\setminus}\bar{G}}\|^2.\]
Recall the fact that, in the backward algorithm,  we will stop at the point where the left 
hand side is no less than $\frac{\delta^{(k)}}{2}$. Hence, we have 
$\| \w^{(k)}_{G^{(k)} {\setminus} \bar{G}}\|^2\geqslant\frac{|G^{(k)} {\setminus} \bar{G}|\delta^{(k)}}{\varphi_{+}(1)}$.
The first two equalities in the lemma are trivial observations, and the last inequality is due to the requirement that $\delta^{(k)}>\lambda\delta$ in our Algorithm \ref{alg:general}. It completes the proof of the lemma.
\end{proof}

The above Lemma $\ref{backward}$ will help us in transforming the sparsity condition into estimating the errors of the parameters.
Moreover, using this idea, we can just
concentrate on the case where our selection of groups is no better
than ``ideal'' sparse solution we assumed as shown below. 

\begin{lemma}
\label{decrease}
For any 
  integer $s$ larger than $|G^{(k)} {\setminus} \bar{G}|$, and if 
  \[\delta>\frac{4\varphi_{+}(1)}{\lambda\varphi_{-}^{2}(s)}\|\nabla Q(\bar{\w})\|_{G,\infty}^{2},\]  
  then at the end of each backward step $k$, we have
\begin{equation}\label{eq:decrease}
  Q(\w^{(k)})\geqslant Q(\bar{\w}).
\end{equation}
\end{lemma}
\begin{proof}[Proof of Lemma~\ref{decrease}]
Considering the difference of the two values, we have 
\begin{align*}
& Q(\w^{(k)})- Q(\bar{\w})&\\
\geqslant&\langle\nabla Q(\bar{\w}),\w^{(k)}-\bar{\w}\rangle+\frac{\varphi_{-}(s)}{2}\|\w^{(k)}-\bar{\w}\|^{2}&\\
		      =& \langle\nabla Q(\bar{\w}),{(\w^{(k)}-\bar{\w})}_{{G^{(k)}\setminus \bar{G}}} \rangle +\frac{\varphi_{-}(s)}{2}\|\w^{(k)}-\bar{\w}\|^{2}&\\
	       \geqslant &-\|\nabla Q(\bar{\w})\|_{G,\infty} \|\w^{(k)}-\bar{\w}\|_{G,1} +\frac{\varphi_{-}(s)}{2}\|\w^{(k)}-\bar{\w}\|^{2}&\\
	       	       \geqslant &-\|\nabla Q(\bar{\w})\|_{G,\infty} \|\w^{(k)}-\bar{\w}\|\sqrt{|G^{(k)} {\setminus} \bar{G}|} +\frac{\varphi_{-}(s)}{2}\|\w^{(k)}-\bar{\w}\|^{2}&(\textrm{Cauchy-Schwartz} )\\
		      \geqslant &-\|\nabla Q(\bar{\w})\|_{G,\infty}
                                  \|\w^{(k)}-\bar{\w}\|^{2}\sqrt{\frac{\varphi_{+}(1)}{\delta^{(k)}}}
                                  +\frac{\varphi_{-}(s)}{2}\|\w^{(k)}-\bar{\w}\|^{2}.&
                                                                                       (\textrm{Lemma~\ref{backward}})
\end{align*}
Finally, from our choice of $\delta$ and the fact that $\delta^{(k)}>\lambda \delta$, we can then conclude that the last line is bigger than $0$.
Hence we will have general assumption $Q(\w^{(k)})\geqslant Q(\bar{\w})$ if 
$\delta>\frac{4\varphi_{+}(1)}{\varphi_{-}^{2}(s)}\|\nabla Q(\bar{\w})\|_{G,\infty}^{2}$.
This completes the proof of the lemma.
\end{proof}

One key part of the proof idea is to compare the difference of the convex function at different optimal values supported on different groups with some other quantities. We have the following.
\begin{lemma}
 \label{forward_process_obj}
Suppose that $\hat{G}$ is a subset of the super set $G$ and we set   $\hat{\w}=\underset{\operatorname{supp}(\w)\subset F_{\hat{G}}}{\operatorname{argmin}} Q(\w). $ For any $\w^{\prime}$ with supporting groups $G^{\prime}$ and corresponding
  features $F^{\prime}=F_{G^{\prime}}$, if
  
  \[g\in \{\bar g \in G^{\prime}{{\setminus}} \hat{G}\ \ | \ \
    Q(\hat{\w})-\underset{\bd\alpha}{\operatorname{min}}\
    Q(\hat{\w}+\mathbf{E}_{\bar g}\bd\alpha)\geqslant \lambda( Q(\hat{\w})-
\underset{\tilde g \in G^{\prime}{{\setminus}} \hat{G}
  ,\bd\alpha}{\operatorname{min}} Q(\hat{\w}+\mathbf{E}_{\tilde g}\bd\alpha))\},\]
then we have

\[Q(\hat{\w})-Q(\w^{\prime})\leqslant \frac{\varphi_{+}(1)|G^{\prime} {\setminus} 
\hat{G}|(Q(\hat{\w})-\underset{\bd\alpha}{\operatorname{min}}Q(\hat{\w}+\mathbf{E}_{g}\bd\alpha))}{\lambda\varphi_{-}(s)}.\] 
\end{lemma}

\begin{proof}[Proof of Lemma~\ref{forward_process_obj}]
We are comparing the differences of the convex function and its value along the following special directions.
\begin{align*}
  & |G^{\prime}{{\setminus}} \hat{G}| \underset{g \in G^{\prime}{{\setminus}} \hat{G}, \eta }{\operatorname{min}}Q({\hat{\w}}+\eta ({\w}_{g}^{\prime}-{\hat{\w}}_{g}))&\\
  \leqslant&\underset{g \in G^{\prime}{{\setminus}} \hat{G}}{\sum}Q({\hat{\w}}+\eta({\w}_{g}^{\prime}-{\hat{\w}}_{g}))&(\textrm{any } \eta)\\
  \leqslant&\underset{g \in G^{\prime}{{\setminus}} \hat{G}}{\sum}\left( Q({\hat{\w}})+\langle\eta({\w}_{g}^{\prime}-{\hat{\w}}_{g}),\nabla Q({\hat{\w}})\rangle  +\frac{\varphi_{+}(1)}{2}\|\eta({\w}_{g}^{\prime}-{\hat{\w}}_{g})\|^{2}\right)& (\textrm{definition of }\varphi_{+}(1) )\\
  \leqslant &|G^{\prime}{{\setminus}} \hat{G}|Q(\hat{\w})+\langle{\eta}(\w^{\prime}-\hat{\w})_{G^{\prime}{\setminus} \hat{G}},\nabla Q({\hat{\w}})\rangle+\frac{\varphi_{+}(1)}{2}\|\eta(\w^{\prime}-\hat{\w})\|^{2}&(\textrm{non-intersection property})\\
  =& |G^{\prime}{{\setminus}} \hat{G}|Q(\hat{\w})+\langle{\eta}(\w^{\prime}-\hat{\w}),(\nabla Q(\hat{\w}))_{G^{\prime}{\setminus} \hat{G}}\rangle+\frac{\varphi_{+}(1)}{2}\|\eta(\w^{\prime}-\hat{\w})\|^{2}&\\
    =& |G^{\prime}{{\setminus}} \hat{G}|Q(\hat{\w})+\langle{\eta}(\w^{\prime}-\hat{\w}),(\nabla Q(\hat{\w}))_{G^{\prime}\cup \hat{G} }\rangle+\frac{\varphi_{+}(1)}{2}\|\eta(\w^{\prime}-\hat{\w})\|^{2}&(\textrm{optimal solution  on } \hat{G})\\
  \leqslant & |G^{\prime}{{\setminus}} \hat{G}|
  Q(\hat{\w})+{\eta}\Bigl(Q(\w^{\prime})-Q(\hat{\w})-\frac{\varphi_{-}(s)}{2}\|\w^{\prime}-\hat{\w}\|^{2}\Bigr)&
  \\
&\qquad +\frac{\varphi_{+}(1)}{2}\|\eta(\w^{\prime}-\hat{\w})\|^{2}.&
\end{align*}
And the special direction indeed gives us some bounds as follows: 

\begin{align*}
  & |G^{\prime}{{\setminus}}\hat{G}| \left(\underset{g \in G^{\prime}{{\setminus}} \hat{G}, \eta }{\operatorname{min}}Q({\hat{\w}}+\eta ({\w}_{g}^{\prime}-{\hat{\w}}_{g}))-Q(\hat{\w})\right)&\\
 \leqslant & \underset{\eta}{\min}\ {\eta}\left(Q(\w^{\prime})-Q(\hat{\w})-\frac{\varphi_{-}(s)}{2}\|(\w^{\prime}-\hat{\w})\|^{2}\right)+\frac{\varphi_{+}(1)}{2}\|\eta(\w^{\prime}-\hat{\w})\|^{2}&\\
=&-\frac{{\left(Q(\w^{\prime})-Q(\hat{\w})-\frac{\varphi_{-}(s)}{2}\|(\w^{\prime}-\hat{\w})\|^{2}\right)}^{2}}{2\varphi_{+}(1)\|(\w^{\prime}-\hat{\w})\|^{2}}&\\
\leqslant&\frac{4(Q(\w^{\prime})-Q(\hat{\w}))\frac{\varphi_{-}(s)}{2}\|(\w^{\prime}-\hat{\w})\|^{2}}{2\varphi_{+}(1)\|(\w^{\prime}-\hat{\w})\|^{2}}&\\
=&-\frac{\varphi_{-}(s)(Q(\hat{\w})-Q(\w^{\prime}))}{\varphi_{+}(1)}.&
\end{align*}
Hence, by rearranging the above inequality, we have 
\begin{align*}
  & |G^{\prime}{{\setminus}} \hat{G}| \left(Q(\hat{\w})- \underset{ \bd\alpha }{\operatorname{min}}\ Q(\hat{\w}+\mathbf{E}_{g}{\bd\alpha})\right)&\\
 \geqslant & |G^{\prime}{{\setminus}} \hat{G}|\left(Q(\hat{\w})- \underset{g \in G^{\prime}{{\setminus}} \hat{G}, \bd\alpha }{\operatorname{min}} \ Q(\hat{\w}+\mathbf{E}_{g}{\bd\alpha})\right)\lambda&\\
  \geqslant & |G^{\prime}{{\setminus}} \hat{G}| \left( Q(\hat{\w})-\underset{{g}\in G^{\prime}{{\setminus}} \hat{G}, \eta }{\operatorname{min}}Q(\hat{\w}+\eta ({\hat{\w}}_{g}-\w^{\prime}_{g}))\right)\lambda& (\textrm{special direction of } \bd\alpha)\\
\geqslant&\frac{\varphi_{-}({s})(Q(\hat{\w})-Q(\w^{\prime}))\lambda}{\varphi_{+}(1)}.&
\end{align*}
It completes the proof.
\end{proof}

In the following, we provide lemmas for analysis of the GIGA algorithm.

\begin{lemma}
\label{lem:giga11}
If for any $\w \in \mathcal{D}$ we have $\|\nabla Q(\w)\|_{G,\infty}\geqslant \lambda \varepsilon$,  then for any $g \in G$, we have 
\[Q(\w)-\underset{\bd\alpha}{\operatorname{min}}Q(\w+\mathbf{E}_{g}\bd\alpha)\geqslant\frac{\lambda^2\varepsilon^{2}}{2{\varphi}_{+}(1)}.\]
\end{lemma}

\begin{proof}[Proof of Lemma~\ref{lem:giga11}]
Since $\w \in D$, we can use the definition of $\varphi_{+}$ as follows:
\begin{align*}
& Q(\w)-\underset{\bd\alpha}{\operatorname{min}}\ Q(\w+\mathbf{E}_{g}\bd \alpha)& \\
\geqslant & \underset{\bd\alpha}{\operatorname{max}}\left( Q(\w)-Q(\w+\mathbf{E}_{g} \bd\alpha) \right)& \\
\geqslant & \underset{\bd\alpha}{\operatorname{max}}\left(-\langle \nabla Q(\w), \mathbf{E}_{g}\bd\alpha\rangle -\frac{{\varphi}_{+}(1)}{2}\| \mathbf{E}_{g}\bd\alpha\|^{2}\right)& \\
\geqslant & \underset{ \bd\alpha}{\operatorname{max}}\left(-\|(  \nabla Q(\w))_{g}\| \|\mathbf{E}_{g}\bd\alpha\| -\frac{{\varphi}_{+}(1)}{2}\|\bd\alpha\|^{2}\right)& \\
\geqslant & \underset{\bd\alpha}{\operatorname{max}}\left(-\|\nabla Q(\w)\|_{G,\infty}\|\bd\alpha\| -\frac{{\varphi}_{+}(1)}{2}\|\bd\alpha\|^{2}\right)& \\
=& \frac{\|\nabla Q(\w)\|^{2}_{G,\infty}}{2{\varphi}_{+}(1)}&\\
\geqslant& \frac{\lambda^2 \varepsilon^{2}}{2{\varphi}_{+}(1)}. 
\end{align*}
Here we just move forward  to the maximal group normal and get a more simplified result. 
Also, the calculation tells us that in our algorithm, 
\begin{align}
\label{deltabound}
Q(\w^{(k)})-\underset{\bd\alpha}{\operatorname{min}}\ Q(\w^{(k)}+\mathbf{E}_{g}\bd\alpha)\geqslant \delta^{(k)}\geqslant\frac{\lambda^2\varepsilon^{2}}{2{\varphi}_{+}(1)}.
\end{align}
 It completes the proof of the lemma.
\end{proof}

\begin{lemma}
\label{backwardgdt}
Consider $\w^{(k)}$ with the support $G^{(k)}$ at the end of our backward 
elimination process in our algorithm.   Then we have 
$\|\w^{(k)}_{G^{(k)}\setminus\bar{G}}\|^{2}= \|(\w^{(k)}-\bar{\w})\|^{2}=\|(\w^{(k)}-\bar{\w})_{G^{(k)}\setminus\bar{G}}\|^{2} \geqslant \frac{\delta^{(k)}|G^{(k)}\setminus\bar{G}|}{{\varphi}_{+}(1)}$.
\end{lemma}
\begin{proof}[Proof of Lemma~\ref{backwardgdt}]
Similar to the objection function case, we derive such relations by regrouping and the proof is almost identical to  Lemma \ref{backward} since we have the same backward method for different algorithms.
\end{proof}

\begin{lemma}
\label{lem:gigacompare}
The optimal value of $Q$ we get from $G^{(k)}$ is bigger than the ``idealized'' optimal sparse value. 
\end{lemma}
\begin{proof}[Proof of Lemma~\ref{lem:gigacompare}]
A direct way to compare these values is taking the difference of the two: 
\begin{align*}
& Q(\w^{(k)})-Q(\bar{\w})&\\
\geqslant & \langle \nabla Q(\bar{\w}),\w^{(k)}-\bar{\w}\rangle +\frac{{\varphi}_{-}(s)\|\w^{(k)}-\bar{\w}\|^{2}}{2}&\\
\geqslant &-\|\nabla Q(\bar{\w})\|_{G,\infty} \|\w^{(k)}-\bar{\w}\|_{G,1} +\frac{{\varphi}_{-}(s)\|\w^{(k)}-\bar{\w}\|^{2}}{2}&\\
\geqslant &-\|\nabla Q(\bar{\w})\|_{G,\infty}\sqrt{ |G^{(k)}\setminus\bar{G}|}\|\w^{(k)}-\bar{\w}\|_{G,2} +\frac{{\varphi}_{-}(s)\|\w^{(k)}-\bar{\w}\|^{2}}{2}&(\textrm{Cauchy-Schwartz})\\
\geqslant &-\frac{\|\nabla Q(\bar{\w})\|_{G,\infty}\sqrt{{\varphi}_{+}(1)}\|\w^{(k)}-\bar{\w}\|^{2}}{\sqrt{\delta^{(k)}}} +\frac{{\varphi}_{-}(s)\|\w^{(k)}-\bar{\w}\|^{2}}{2}&\\
\geqslant & \left( \frac{{\varphi}_{-}(s)}{2} -\frac{\|\nabla Q(\bar{\w})\|_{G,\infty}\sqrt{2}{\varphi}_{+}(1)}{\lambda \varepsilon} \right) \|\w^{(k)}-\bar{\w}\|^{2}.&
\end{align*}
Note we have the bound for the threshold  $\delta^{(k)}$ as in (\ref{deltabound}). And if we take 
$\varepsilon \geqslant \frac{2\sqrt{2}{\varphi}_{+}(1)\|\nabla Q(\bar{\w})\|_{G,\infty}}{\lambda{\varphi}_{-}(s)}$, then we can show that the expression above is no less than $0$. It completes the proof of the lemma.
\end{proof}

The following Lemma~\ref{boundofiteration_gdt} is useful in proving Theorem~\ref{convergency_gdt}. As before, the key part of analysis is to compare the relation of difference of loss functions in our selection process and current optimal
at various states. We define a set $M$ as 
$\{g \in G^{\prime}\setminus \bar{G}:\|(\nabla_{} Q(\bar{\w}))_{g}\|\geqslant \lambda \|\nabla Q(\bar{\w})\|_{G,\infty} \}$.
Recall that $\bar{\w}=\underset{\textrm{supp }(\w)\subset F_{\bar{G}}}{\operatorname{argmin}}\ Q(\w)$ and 
for any $\w^{\prime}$ with support group $G^{\prime}$. We will 
just use $G^{\prime}$ as $G^{(k)}\bigcup\bar{G}$.  
\begin{lemma}
  \label{boundofiteration_gdt}
If we take any $g\in M$, we will have the following
\begin{equation*}
  | G^{\prime}\setminus \bar{G}|\left( Q(\bar{\w})-\underset{\bd\alpha}{\operatorname{min}}Q(\bar{\w}+\mathbf{E}_{g}\bd\alpha)\right)\geqslant \frac{{\varphi}_{-}(s)\lambda^2}{{\varphi}_{+}(1)}(Q(\bar{\w})-Q(\w^{\prime}))
\end{equation*}
\end{lemma}

\begin{proof}[Proof of Lemma~\ref{boundofiteration_gdt}]
Let $\mathcal{W}_g$  be the subspace of $\R^p$ spanned by column vectors of $\mathbf{E}_{g}$ where $g\in G$.   For any
${f_{g}}\in \mathcal{W}_g$, we define 
\begin{equation*}
P_{f_{g}}(\eta)=\langle \nabla Q(\bar{\w}),\eta f_{g}\rangle+\frac{{\varphi}_{+}(1)}{2}\eta^{2}\|f_{g}\|^{2}.
\end{equation*}
In fact, we just replace $\mathbf{E}_{g}\bd\alpha$ by $\eta f_g$. But this further separation of $\mathbf{E}_{g}\alpha$ in terms of direction $f_g$ and length $\eta$ will provide more freedom for our calculation.
We further define 
\begin{equation}\label{minidirection}
P_{g}(\eta)=\underset{f_{g}\in \mathcal{W}_g,\|f_{g}\|=1}{\operatorname{min}}P_{f_{g}}(\eta).
\end{equation}
So this new function is searching a direction such that $P_{f_{g}}$ can get a minimum value of a fixed length.
One can observe that 
\begin{equation*}
\underset{\eta}{\operatorname{min}}\ P_{g}(\eta)=-\frac{\|{\nabla}_{g}
Q(\bar{\w})\|^{2}}{2{\varphi}_{+}(1)}.
\end{equation*}

Taking $u=\|\w^{\prime}_{G^{\prime}\setminus \bar{G}}\|_{G,1}=\|(\w^{\prime}-\bar{\w})_{G^{\prime}\setminus \bar{G}}\|_{G,1}$, we will have 
\begin{align*}
  & u \underset{g\in G^{\prime}\setminus \bar{G}}{\operatorname{min}}P_{g}(\eta)\\
\leqslant& \underset{g\in G^{\prime}\setminus \bar{G}}{\operatorname{\sum}}\|\w^{\prime}_{g}\|P_{g}(\eta)\\
  \leqslant& \underset{g\in G^{\prime}\setminus \bar{G}}
{\operatorname{\sum}}\|\w^{\prime}_{g}\|P_{\frac{\w^{\prime}_{g}}{\|\w^{\prime}_{g}\|}}(\eta) &(\textrm{due to } (\ref{minidirection}))\\
  =& \underset{{g}\in {G^{\prime}\setminus \bar{G}}}{\operatorname{\sum}}\left(\eta 
  \langle \nabla  Q(\bar{\w}),\frac{\w^{\prime}_{g}}{\|\w^{\prime}_{g}\|}\|\w^{\prime}_{g}\|\rangle \right)+\frac{u{\varphi}_{+}(1)\eta^{2}}{2}\\
  =& \langle\nabla Q(\bar{\w}),\w^{\prime}_{{G^{\prime}\setminus \bar{G}}}\rangle \eta+\frac{u{\varphi}_{+}(1)\eta^{2}}{2}\\
  =& \langle\nabla Q(\bar{\w}),(\w^{\prime}-\bar{\w})_{{G^{\prime}\setminus \bar{G}}}\rangle \eta+\frac{u{\varphi}_{+}(1)\eta^{2}}{2}\\
  =& \langle\nabla Q(\bar{\w}),\w^{\prime}-\bar{\w}\rangle \eta+\frac{u{\varphi}_{+}(1)\eta^{2}}{2}\\
  \leqslant& \eta\left( Q(\w^{\prime})-Q(\bar{\w})-\frac{{\varphi}_{-}(s)}{2}\|\w^{\prime}-\bar{\w}\|^2\right)+\frac{u{\varphi}_{+}(1)\eta^{2}}{2}.
\end{align*}
Now we take a special value 
$\eta_0=-\frac{\left( Q(\w^{\prime})-Q(\bar{\w})-\frac{{\varphi}_{-}(s)}{2}\|\w^{\prime}-\bar{\w}\|^2\right)}{u{\varphi}_{+}(1)}$ on the above inequality and move the left $u$ to the right, then we will have

\begin{align*}
  & \underset{g\in G^{\prime}\setminus \bar{G}}{\operatorname{min}}P_{g}(\eta_0)&\\
  =&-\frac{\left( Q(\w^{\prime})-Q(\bar{\w})-\frac{{\varphi}_{-}(s)}{2}\|\w^{\prime}-\bar{\w}\|^2\right)^2}{2u^2{\varphi}_{+}(1)}&\\
   \leqslant &-\frac{4\left( Q(\w^{\prime})-Q(\bar{\w})\right)\frac{{\varphi}_{-}(s)}{2}\|\w^{\prime}-\bar{\w}\|^2}{2u^2{\varphi}_{+}(1)}& (\text{due to}~(a-b)^2\geqslant -4ab)\\
     \leqslant &-\frac{\left( Q(\w^{\prime})-Q(\bar{\w})\right){{\varphi}_{-}(s)}}{|G^{\prime}\setminus\bar{G}|{\varphi}_{+}(1)}.&
\end{align*}
Notice that in the last line we take advantage of 
$\|\bar{\w}-\w^{\prime}\|_{G,1}\leqslant\sqrt{|G^{\prime}\setminus\bar{G}|}\|\bar{\w}-\w^{\prime}\|.$ And in short the above inequalities give us \begin{equation}\label{eq:special}
\underset{g\in G^{\prime}\setminus \bar{G}}{\operatorname{min}}P_{g}(\eta_0)\leqslant -\frac{\left( Q(\w^{\prime})-Q(\bar{\w})\right){{\varphi}_{-}(s)}}{|G^{\prime}\setminus\bar{G}|{\varphi}_{+}(1)},
\end{equation}
which will be used in the following calculations.
Combining all the above inequalites, for any $g\in M$, we have the following:
\begin{align*}
&\underset{\bd\alpha}{\operatorname{min}}\ Q(\bar{\w}+\mathbf{E}_{g}\bd\alpha)-Q(\bar{\w})\\
= &\underset{\eta,f_{g}\in   W_{g}, \|f_{g}\|=1}{\operatorname{min}}Q(\bar{\w}+\eta f_{g})-Q(\bar{\w})\\
\leqslant &\underset{\eta,f_{g}\in   W_{g},\|f_{g}\|=1}{\operatorname{min}}\langle \nabla Q(\bar{\w}),\eta f_{g} \rangle+ \frac{{\varphi}_{+}(1)}{2}\eta^{2}\\
=&\underset{\eta}{\operatorname{min}}\ P_{g}(\eta)\\
  =&-\frac{\|{\nabla}_{g}Q(\bar{\w})\|^{2}}{2{\varphi}_{+}(1)}&\\
  \leqslant& \lambda^2 \underset{g^{\prime}\in G^{\prime}\setminus \bar{G},\eta}{\operatorname{min}}P_{g^{\prime}}(\eta)&\\
  \leqslant&  \underset{g^{\prime}\in G^{\prime}\setminus \bar{G}}{\operatorname{min}}P_{g^{\prime}}(\eta_0)\lambda^2&\\
  \leqslant&- \frac{\left( Q(\bar{\w})-Q(\w^{\prime})\right){\varphi}_{-}(s)\|\bar{\w}-\w^{\prime}\|^{2}\lambda^2}{{\varphi}_{+}(1)u^{2}}&(\text{due to}~\eqref{eq:special})\\
\leqslant&- \frac{\left( Q(\bar{\w})-Q(\w^{\prime})\right){\varphi}_{-}(s)\lambda^2}{{\varphi}_{+}(1)|G^{\prime}\setminus\bar{G}|}&
\end{align*}
 It completes the proof of the lemma.
\end{proof}

\end{document}